\pdfoutput=1
\documentclass{article} 
\usepackage{nips15submit_e,times}

\usepackage{url}
\usepackage[utf8]{inputenc}
\usepackage{amsmath}
\usepackage{graphicx}
\usepackage{upgreek}
\usepackage{amsfonts}
\usepackage{amssymb}
\usepackage{amsthm}
\usepackage[mathscr]{euscript}
\usepackage{mathtools}
\usepackage[numbers]{natbib}
\newtheorem{thm}{Theorem}

\newtheorem{lem}{Lemma}
\usepackage{xcolor}
\usepackage{xr}
\usepackage{chngcntr}
\usepackage[page, header]{appendix}
\usepackage{titletoc}
\usepackage{enumitem}
\setlist[itemize]{leftmargin=*}
\setlist[enumerate]{leftmargin=*}

\usepackage{tikz, pgf,pgfplots}
\usetikzlibrary{calc,3d, patterns}
\usetikzlibrary{arrows,shapes,automata}
\usetikzlibrary{decorations.markings, positioning, fit}

\definecolor{DarkRed}{rgb}{0.75,0,0}
\definecolor{DarkGreen}{rgb}{0,0.5,0}
\definecolor{DarkPurple}{rgb}{0.5,0,0.5}
\definecolor{Dark}{rgb}{0.5,0.5,0}
\definecolor{DarkBlue}{rgb}{0,0,0.7}
\usepackage[bookmarks, colorlinks=true, plainpages = false, citecolor = DarkBlue, urlcolor = blue, filecolor = black, linkcolor =DarkGreen]{hyperref}
\usepackage[ruled, vlined]{algorithm2e}

\DeclareMathOperator*{\argmax}{arg\,max}

\allowdisplaybreaks[2]

\newcommand{\prob}{\mathbb P}
\newcommand{\Var}{\mathbb V}
\newcommand{\varV}{\mathscr V}
\newcommand{\indicator}[1]{\mathbb I\{ #1 \} }
\newcommand{\statespace}{\mathcal S}
\newcommand{\actionspace}{\mathcal A}
\newcommand{\saspace}{\statespace \times \actionspace}
\newcommand{\kispace}{\mathcal K \times \mathcal I}
\newcommand{\numki}{\left|\kispace\right|}
\newcommand{\numsa}{\left|\saspace\right|}
\newcommand{\numS}{|\statespace|}
\newcommand{\numA}{|\actionspace|}
\newcommand{\numSucc}[1]{|\statespace(#1)|}
\newcommand{\succS}[1]{\statespace(#1)}
\newcommand{\maxNumSucc}{C}

\newcommand{\maxUpdates}{U_{\max}}
\newcommand{\reals}{\mathbb R}

\newcommand{\defeq}{:=}
\usepackage{xspace}
\newcommand{\algname}{\texttt{UCFH}\xspace}

\SetCommentSty{mycommfont}
\mathtoolsset{showonlyrefs=true}

\title{Sample Complexity of Episodic Fixed-Horizon Reinforcement Learning}

\author{
Christoph Dann\\
Machine Learning Department\\
Carnegie Mellon University\\
\texttt{cdann@cdann.net}\\
\And
Emma Brunskill \\
Computer Science Department\\
Carnegie Mellon University \\
\texttt{ebrun@cs.cmu.edu} \\
}

\nipsfinalcopy 

\begin{document}

\maketitle

\begin{abstract}
Recently, there has been significant progress in
understanding reinforcement learning in discounted infinite-horizon Markov decision processes
(MDPs) by deriving tight sample complexity bounds.  However, in many real-world
applications, an interactive learning agent operates for a fixed or bounded
period of time, for example tutoring students for exams or handling customer
service requests. Such scenarios can often be better treated as episodic
fixed-horizon MDPs, for which only looser bounds on the sample complexity exist.
A natural notion of sample complexity in this setting is the number of episodes
required to guarantee a certain performance with high probability (PAC
guarantee). In this paper, we derive an upper PAC bound $\tilde O(\frac{\numS^2
\numA H^2}{\epsilon^2} \ln\frac 1 \delta)$ and a lower PAC bound $\tilde
\Omega(\frac{\numS \numA H^2}{\epsilon^2} \ln \frac 1 {\delta + c})$ that match
up to log-terms and an additional linear dependency on the number of states
$\numS$.  The lower bound is the first of its kind for this setting. Our
upper bound leverages Bernstein's inequality to improve on previous bounds for episodic
finite-horizon MDPs which have a time-horizon dependency of at least $H^3$.
\end{abstract}
\vspace{-5mm}
\section{Introduction and Motivation}
\vspace{-2mm}
\label{sec:intro}

Consider test preparation software that tutors students for a national advanced
placement exam taken at the end of a year, or maximizing business revenue by
the end of each quarter. Each individual task instance requires making a
sequence of decisions for a fixed number of steps $H$ (e.g., tutoring one
student to take an exam in spring 2015 or maximizing revenue for the end of the
second quarter of 2014).  Therefore, they can be viewed as a finite-horizon
sequential decision making under uncertainty problem, in contrast to an
infinite horizon setting in which the number of time steps is infinite. 
When the domain parameters (e.g. Markov decision process parameters) 
are not known in advance, and there is the opportunity to 
repeat the task many times (teaching a new student for 
each year's exam, maximizing revenue for each new quarter), 
 this can be treated as episodic fixed-horizon 
 reinforcement learning (RL).
One important question is to understand how much experience 
is required to act well in this setting. We formalize this as the sample 
complexity of reinforcement learning \citep{Strehl2006}, 
which is the number of time steps on which the algorithm 
may select an action whose value is not near-optimal. RL 
algorithms with a sample complexity that is a polynomial 
function of the domain parameters are referred to as Probably 
Approximately Correct (PAC)~\citep{Kearns1999,Brafman2003,Kakade2003,Strehl2006}. 
Though there has been significant work on PAC RL algorithms for 
the infinite horizon setting, there has been relatively 
little work on the finite horizon scenario.

In this paper we present the first, to our knowledge, lower bound, and a 
new upper bound  on the sample complexity 
of episodic finite horizon PAC 
reinforcement learning in discrete state-action spaces. Our  
bounds are tight up to log-factors in the
time horizon $H$, the accuracy $\epsilon$, the number of actions $\numA$ and up
to an additive constant in the failure probability $\delta$. These bounds
improve upon existing results by a factor of at least $H$. 
Our results also apply when the reward model is a function of the within-episode time step in addition 
to the state and action space. While we assume a stationary transition 
model, our results can be extended readily to time-dependent state-transitions.
Our proposed \algname (Upper-confidence fixed-horizon RL) algorithm that achieves our upper PAC guarantee can be
applied directly to wide range of fixed-horizon episodic MDPs with known
rewards.\footnote{\label{fn:known_rewards} Previous
    works \citep{AuerOrtner2005} have shown that the complexity of learning
    state transitions usually dominates learning reward functions.  We
    therefore follow existing sample complexity analyses
    \citep{Lattimore2012,Szita2010} and assume known rewards for simplicity.
    The algorithm and PAC bound can be extended readily to the case of unknown reward functions.
} 
It does not require additional structure such as assuming access to a  
generative model~\citep{Azar2012} or that the state transitions 
are sparse or acyclic~\citep{Lattimore2012}.

The limited prior research on upper bound PAC results for finite horizon
MDPs has focused on different settings, such as
partitioning a longer trajectory into fixed length
segments~\citep{Kakade2003,Strehl2006}, or considering a sliding time
window~\citep{Kolter2009a}. The tightest dependence on the horizon in
terms of the number of episodes presented
in these approaches is at least $H^3$ whereas our dependence is only $H^2$. More
importantly, such alternative settings require the optimal policy to be
stationary, whereas in general in finite horizon settings the optimal policy
is nonstationary (e.g. is a function of both the state and the within 
episode time-step).\footnote{The best action will generally 
depend on the state and the number of remaining time steps. 
In the tutoring example, even if the student has the same 
state of knowledge, the optimal tutor decision may be to space practice if there is many days till the test and 
provide intensive short-term practice if the test is tomorrow.} 
Fiechter~\citep{Fiechter1994,Fiechter1997} and
\citet{Reveliotis2007} do tackle a closely related setting, but find a
dependence that is at least $H^4$. 

Our work builds on recent work~\citep{Lattimore2012,Azar2012} on PAC
infinite horizon discounted RL that offers much tighter 
upper and lower sample complexity
bounds than was previously known.
To use an infinite horizon algorithm in a finite 
horizon setting,  
a simple 
change is to augment the state space by the time 
step (ranging over $1,\dots,H$), which enables 
the learned policy to be non-stationary in the original 
state space (or 
equivalently, stationary in the newly augmented 
space). Unfortunately, since these recent bounds 
are in general a quadratic function of the state space 
size, the proposed state space expansion 
would introduce at least an additional $H^2$ factor 
in the sample complexity term, yielding at least a 
$H^4$ dependence in the number of episodes for 
the sample complexity. 

Somewhat surprisingly, we prove an 
upper bound on the sample complexity for the finite horizon 
case that only scales quadratically with the horizon. 
A key part of our proof is that the variance of the value 
function in the finite horizon setting satisfies a 
Bellman equation.  
We also leverage recent insights that state--action 
pairs can be estimated to different precisions depending 
on the frequency to which they are visited under a policy, 
extending these ideas to also handle when the policy 
followed is nonstationary. Our lower bound analysis 
is quite different than some prior infinite-horizon 
results, and involves a construction of parallel multi-armed 
bandits where it is required that the best arm in a certain portion of 
the bandits is identified with high probability to achieve near-optimality. 

\section{Problem Setting and Notation}
\label{sec:problem_setting}
We consider episodic 
fixed-horizon MDPs, which can be formalized as
a tuple $M = (\statespace, \actionspace, r, p, p_0, H)$. Both, the statespace
$\statespace$ and the actionspace $\actionspace$ are finite sets. The learning
agent interacts with the MDP in episodes of $H$ time steps.  At time $t=1 \dots
H$, the agent observes a state $s_t$ and choses an action $a_t$ based on a
policy $\pi$ that potentially depends on the within-episode time step, 
i.e., $a_t = \pi_t(s_t)$ for $t=1, \dots, H$. The
next state is sampled from the stationary transition kernel $s_{t+1} \sim
p(\cdot | s_t, a_t)$ and the initial state from $s_1 \sim p_0$. In addition the agent receives a reward drawn 
from a distribution\footnote{It is straightforward to have the reward 
depend on the state, or state/action or state/action/next state.} 
with mean $r_t(s_t)$ 
determined by the reward function.  The reward function $r$ is possibly
time-dependent and takes values in $[0,1]$.  The quality of a policy $\pi$ is
evaluated by the \emph{total expected reward} of an episode $R^\pi_M = \mathbb
E \left[\sum_{t=1}^H r_t(s_t) \right]$. For simplicity,$^{\ref{fn:known_rewards}}$
we assume that the reward function $r$ is known to the agent but
the transition kernel $p$ is unknown.
The question we study is how many episodes does a learning agent
follow a policy $\pi$ that is not $\epsilon$-optimal, i.e., $R^*_M - \epsilon >
R^{\pi}_M$, with probability at least $1 - \delta$ for any chosen accuracy
$\epsilon$ and failure probability $\delta$.

\paragraph{Notation.}
In the following sections, we reason about
the true MDP $M$, an empirical MDP $\hat M$ and an optimistic MDP $\tilde M$
which are identical except for their transition probabilities $p$, $\hat p$ and
$\tilde p_t$.  We will provide more details about these MDPs later. We introduce the notation explicitly only for $M$ but the
quantities carry over to $\tilde M$ and $\hat M$ with additional tildes or hats by replacing $p$ with $\tilde p_t$ or $\hat p$. 

The (linear) operator $P^\pi_i f(s) \defeq \mathbb E[ f(s_{i+1}) | s_i = s] =
\sum_{s' \in \statespace} p(s' | s, \pi_i(s)) f(s')$ takes any function $f: \statespace \rightarrow \reals$
and returns the expected value of $f$ with respect to the next
time step.\footnote{The definition also works for time-dependent transition
probabilities.}
For convenience, we define the multi-step version as $P^\pi_{i:j} f \defeq P^\pi_i P^\pi_{i+1} \dots P^\pi_j f$. 
The value function from time $i$ to time $j$ is defined as
$V^\pi_{i:j}(s) \defeq \mathbb E\left[ \sum_{t=i}^j r_t(s_t) | s_i = s \right]
= \sum_{t=i}^j P^\pi_{i:t-1} r_t  = \left(P^\pi_i V^\pi_{i+1:j}\right)(s) + r_i(s)$
and $V^*_{i:j}$ is the optimal value-function.
When the policy is clear, we omit the superscript $\pi$.

We denote by $\succS{s,a} \subseteq \statespace$ the set of possible successor
states of state $s$ and action $a$. The maximum number of them is denoted by
$\maxNumSucc = \max_{s,a \in \saspace} \numSucc{s,a}$. In general, without making
further assumptions, we have $\maxNumSucc = \numS$, though in many 
practical domains (robotics, user modeling) each state can only 
transition to a subset of the full set of states (e.g. a robot can't 
teleport across the building, but can only take local moves). 
The notation $\tilde O$ is similar to the usual $O$-notation but ignores
log-terms. More precisely $f = \tilde O(g)$ if there are constants $c_1$,
$c_2$ such that $f \leq c_1 g (\ln g)^{c_2}$ and analogously for $\tilde \Omega$.
The natural logarithm is $\ln$ and $\log = \log_2$ is the base-2 logarithm.

\section{Upper PAC-Bound}
We now introduce a new model-based algorithm, \algname, for RL in finite horizon 
episodic domains. We will later prove \algname  
is PAC with an upper bound 
on its sample complexity that is smaller than prior approaches. Like many other PAC RL
algorithms~\citep{Brafman2003,Strehl2006a,Strehl2009,Auer2009}, \algname uses
an optimism under uncertainty approach to balance exploration and exploitation.
The algorithm generally works in phases comprised of optimistic
planning, policy execution and model updating that take several episodes each. Phases are indexed by $k$.
As the agent acts in the environment and observes $(s,a,r,s')$ tuples, \algname maintains a confidence set 
over the possible transition parameters for each state-action pair that are 
consistent with the observed transitions. 
Defining such a confidence set that holds with high probability 
can be be achieved using concentration inequalities like the Hoeffding inequality. 
One innovation in our work is to use a particular new set of 
conditions to define the confidence set  that enables us to obtain our tighter 
bounds. We will discuss the confidence sets further below. 
The collection of these confidence sets together form a class 
of MDPs $\mathcal M_k$ that are consistent with the observed data. 
We define  $\hat M_k$ as the maximum likelihood estimate of the MDP 
given the previous observations. 

Given $\mathcal M_k$,  \algname computes a policy $\pi^k$ by performing 
optimistic planning. Specifically, we use a finite horizon variant of extended value iteration (EVI)~\cite{AuerOrtner2005,Strehl2009}. 
EVI performs modified Bellman backups that are optimistic with respect to a given set of parameters. 
That is, given a confidence set of possible transition model parameters, it selects in each time step the model within that set that 
maximizes the expected sum of future rewards. Appendix~\ref{sec:fhevi} provides more details about fixed horizon EVI.

\algname then executes $\pi^k$ until there
is a state-action pair $(s,a)$ that has been visited often enough  since its
last update (defined 
precisely in the until-condition in \algname). After updating the model statistics for this $(s,a)$-pair, a new
policy $\pi^{k+1}$ is obtained by optimistic planning again.  We refer to each
such iteration of planning-execution-update as a \emph{phase} with index $k$. If
there is no ambiguity, we omit the phase indices $k$ to avoid cluttered
notation.

\begin{algorithm}[t]
\SetKwFunction{confset}{ConfidenceSet} 
\SetKwFunction{evi}{FixedHorizonEVI} 
\SetKwFunction{sampleep}{SampleEpisode} 
\SetKwFunction{update}{UpdateModel} 
\SetKwProg{myproc}{Procedure}{}{}
\SetKwProg{myfun}{Function}{}{}
\SetKwInOut{Input}{Input}
\Input{desired accuracy $\epsilon \in (0,1]$, failure tolerance $\delta \in (0,1]$, fixed-horizon MDP $M$}
\KwResult{with probability at least $1 - \delta$: $\epsilon$-optimal policy}
$k\defeq 1$,\hspace{5mm}  $w_{\min} \defeq \frac{\epsilon}{4H\numS}$,\hspace{5mm} 
$\delta_1 \defeq \frac{\delta}{2 \maxUpdates \maxNumSucc}$,\hspace{5mm} $U_{\max} \defeq \numsa \log_2 \frac{\numS H}{w_{\min}}$\;
$m \defeq 
 512 (\log_2 \log_2 H)^2 \frac{\maxNumSucc H^2}{\epsilon^2} \log^2 \left(\frac{8 H^2 \numS^2}{\epsilon} \right) 
        \ln \frac{6 \numsa \maxNumSucc \log^2_2 (4 \numS^2 H^2 / \epsilon)}{\delta}
$\;
$n(s,a) = v(s,a) = n(s,a,s') \defeq 0 \quad \forall, s \in \statespace, a \in \actionspace, s' \in \succS{s,a}$\;
\While{}{
    \tcc{Optimistic planning}
    $\hat p(s' | s,a) \defeq n(s, a, s') / n(s, a)$, for all $(s,a)$ with $n(s,a) > 0$ and $s' \in \succS{s,a}$\;
    $\mathcal M_k \defeq \big\{ \tilde{M} \in \mathcal M_{\textrm{nonst.}} \, : \, \forall (s,a) \in \saspace, t =1\dots H, s' \in \succS{s,a}$\\
        $\qquad\qquad\qquad{\tilde p}_{t}(s'|s,a) \in \confset{${\hat p}(s' |s,a), n(s, a)$} \big\}
    $\;
    $\tilde M_k, \pi^k \defeq \evi{$\mathcal M_k$}$\;
    \tcc{Execute policy}
    \Repeat{there is a $(s,a) \in \saspace$ with $v(s, a) \geq \max\{ m w_{\min}, n(s, a) \}$ and
    $n(s, a) < \numS m H$}
    {\sampleep{$\pi^k$} \tcp*[l]{from $M$ using $\pi^k$}}
    \tcc{Update model statistics for one $(s,a)$-pair with condition above}
    $n(s, a) \defeq n(s, a) + v(s, a)$\;
    $n(s, a, s') \defeq n(s, a, s') + v(s, a, s') \quad \forall s' \in \succS{s,a}$\;
    $v(s, a) \defeq v(s, a, s') \defeq 0 \quad \forall s' \in \succS{s,a}$; $k \defeq k+1$

}
\myproc{\sampleep{$\pi$}}{
    $s_0 \sim p_0$\; 
    \For{$t=0$ \KwTo $H-1$}
    {
        $a_t \defeq \pi_{t+1}(s_{t})$ and $s_{t+1} \sim p( \cdot | s_{t},a_t)$\;
        $v(s_t,a_t) \defeq v(s_t,a_t) + 1$ and $v(s_t,a_t,s_{t+1}) \defeq v(s_t,a_t,s_{t+1}) + 1$\;
    }
}

\myfun{\confset{$p$, $n$}}{
    \vspace{-3mm}
\begin{align}
    \mathcal P \defeq \bigg\{ p' \in [0,1] :& 
    \textrm{if } n > 1: \left|\sqrt{p' (1-p')} - \sqrt{p (1 - p)}\right| \leq \sqrt{\frac{2 \ln(6 / \delta_1)}{n-1}} ,
    \label{eqn:conf_set_var}\\
    & |p - p'| \leq \min \left( \sqrt{\frac{\ln(6 / \delta_1)}{2 n}},
\sqrt{\frac{2 p(1-p)}{n}\ln(6 / \delta_1)} + \frac{7}{3(n-1)}\ln\frac {6}{\delta_1}\right)\bigg\}
    \label{eqn:conf_set_bern}
\end{align}
    \KwRet{$\mathcal P$}
\vspace{-1mm}
}
\caption{\algname: {\bf U}pper-{\bf C}onfidence {\bf F}ixed-{\bf H}orizon episodic reinforcement learning algorithm}
\label{alg:fhalg}
\end{algorithm}

\algname is inspired by the infinite-horizon 
\texttt{UCRL}-$\gamma$ algorithm by \citet{Lattimore2012}
but has several important differences. First, the policy can only be updated at
the end of an episode, so there is no need for explicit delay phases as in
\texttt{UCRL}-$\gamma$. Second, the policies $\pi^k$ in \algname are
time-dependent. 
Finally, \algname can directly deal with non-sparse transition probabilities,
whereas \texttt{UCRL}-$\gamma$ only directly allows two possible successor states for each
$(s,a)$-pair ($\maxNumSucc = 2$).

\paragraph{Confidence sets.} The class of MDPs $\mathcal M_k$ consists of
fixed-horizon MDPs $M'$ with the known true reward function $r$ and where the transition probability $p'_t(s' | s,a)$ from any $(s,a) \in \saspace$ to $s'
\in \succS{s, a}$ at any time $t$ is in the confidence set induced by $\hat p(s' |
s, a)$ of the empirical MDP $\hat M$. Solely for the purpose of computationally more efficient optimistic planning, we allow time-dependent transitions (allows choosing different transition models in different time steps to maximize reward), but this does not affect the theoretical guarantees as the true stationary MDP is still in $\mathcal M_k$ with high probability.  Unlike the confidence intervals used by
\citet{Lattimore2012}, we not only include conditions based on Hoeffding's
inequality\footnote{The first condition in the $\min$ in Equation~\eqref{eqn:conf_set_bern} is actually not
    necessary for the theoretical results to hold. It can be removed and all
$6/\delta_1$ can be replaced by $4 / \delta_1$.}
 and Bernstein's inequality
 (Eq.~\refeq{eqn:conf_set_bern}), but also require that the standard deviation $\sqrt{p(1-p)}$ of
the Bernoulli random variable associated with this transition is close to the
empirical one (Eq.~\refeq{eqn:conf_set_var}). This additional condition (Eq.~\refeq{eqn:conf_set_var})  is key
for making the algorithm directly applicable to generic MDPs 
(in which states can transition to any number of next states, e.g. $\maxNumSucc > 2$)
while only having a linear dependency on $\maxNumSucc$ in the PAC bound.
\vspace{-3mm}
\subsection{PAC Analysis}
\vspace{-2mm}
For simplicity we assume that each episode starts in a
fixed start state $s_0$. This assumption is not crucial and can easily be removed by
additional notational effort. 
\begin{thm}
    For any $0 < \epsilon, \delta \leq 1$, the following holds.
    With probability at least $1 - \delta$, \algname $\,$  produces a sequence of policies $\pi^k$, that yield at most
    \[
        \tilde O\left( \frac{H^2 \maxNumSucc| \saspace|}{\epsilon^2} \ln \frac{1}{\delta} \right)
    \]
    episodes with $R^* - R^{\pi^k} = V^*_{1:H}(s_0) - V^{\pi^k}_{1:H}(s_0) > \epsilon$. The maximum number of possible successor states is denoted by $1 < \maxNumSucc \leq \numS$.
\label{thm:upper_bound}
\end{thm}

\paragraph{Similarities to other analyses.} 
The proof of Theorem~\ref{thm:upper_bound} is quite long and involved, but 
builds on similar techniques 
for sample-complexity bounds in reinforcement learning (see e.g. 
\citet{Brafman2003,Strehl2008}). The general proof
strategy is closest to the one of \texttt{UCRL}-$\gamma$ \citep{Lattimore2012} and the
obtained bounds are similar if we replace the time horizon $H$ with the
equivalent in the discounted case $1/(1 - \gamma)$. However, there are important 
differences that we highlight now briefly.
\begin{itemize}
    \item A central quantity in the analysis by \citet{Lattimore2012} is the
        local variance of the value function. The exact definition for the
        fixed-horizon case will be given below. The key insight for the almost tight
        bounds of \citet{Lattimore2012} and \citet{Azar2012} is to leverage the fact that these local variances
        satisfy a Bellman equation \citep{Sobel1982}  and so the discounted sum
        of local variances can be bounded by $O((1 - \gamma)^{-2})$ instead of
         $O((1-\gamma)^{-3})$. We prove in
        Lemma~\ref{lem:varV_bellman} that local value function variances
        $\sigma^2_{i:j}$ also satisfy a Bellman equation for fixed-horizon
        MDPs even if transition probabilities and rewards
        are time-dependent. This allows us to bound the total sum of local
        variances by $O(H^2)$ and obtain similarly strong results in this
        setting.
    \item \citet{Lattimore2012} assumed there are only two possible successor
        states (i.e., $\maxNumSucc = 2$) which allows them to easily relate the
        local variances $\sigma_{i:j}^2$ to the difference of the expected
        value of successor states in the true and optimistic MDP $(P_i - \tilde
        P_i) \tilde V_{i+1:j}$. For $\maxNumSucc > 2$, the relation
        is less clear, but we address 
        this by proving a bound with tight dependencies on $\maxNumSucc$ (Lemma~\ref{lem:sigma_bound_generic}). 
    \item To avoid super-linear dependency on $\maxNumSucc$ in the final PAC
        bound, we add the additional condition in
        Equation~\eqref{eqn:conf_set_var} to the confidence set. We show that
        this allows us to upper-bound the total reward difference $R^* -
        R^{\pi^k}$ of policy $\pi^k$ with terms that either depend on
        $\sigma^2_{i:j}$ or decrease linearly in the number of samples. This
        gives the desired linear dependency on $\maxNumSucc$ in the final
        bound.  We therefore avoid assuming $\maxNumSucc = 2$ which makes
        \algname directly applicable to generic MDPs with $\maxNumSucc >2$ without the
        impractical transformation argument used by \citet{Lattimore2012}.
\end{itemize}

We will now introduce the notion of \emph{knownness} and \emph{importance} of
state-action pairs that is essential for the analysis of \algname and
subsequently present several lemmas necessary for the proof of
Theorem~\ref{thm:upper_bound}. We only sketch proofs here
but detailed proofs for all results are available in the appendix.

\paragraph{Fine-grained categorization of $(s,a)$-pairs.}
Many PAC RL sample complexity proofs \citep{Brafman2003,Kakade2003,Strehl2006a,Strehl2009} 
 only have a binary notion of ``knownness'', 
distinguishing between known
(transition probability estimated sufficiently accurately) 
and unknown $(s,a)$-pairs. However, as recently shown 
by \citet{Lattimore2012} for the infinite horizon setting, 
it is possible to obtain much tighter sample complexity results 
by using a more fine grained categorization. In particular, 
a key idea is that in order to obtain accurate estimates 
of the value function of a policy from a starting state, 
it is sufficient to have only a 
loose estimate of the parameters of $(s,a)$-pairs that are unlikely to be 
visited under this policy. 

Let 
the \emph{weight} of a $(s,a)$-pair given policy $\pi^k$ be its 
expected frequency in an episode
\begin{align}
    w_k(s,a) \defeq \sum_{t=1}^{H} \prob(s_t = s, \pi^k_t(s_t) = a) = \sum_{t=1}^{H} P_{1:t-1} \indicator{s = \cdot, a = \pi^k_t(s)} (s_0).
\end{align}
\vspace{-1mm}
The \emph{importance} $\iota_k$ of $(s,a)$ is its relative weight
compared to $w_{\min} \defeq \frac{\epsilon}{ 4 H \numS}$ on a log-scale
\begin{align}
    \iota_k(s,a) \defeq \min \left\{ z_i \, : \, z_i \geq \frac{w_k(s,a)}{w_{\min}} \right\} 
    \quad \textrm{where } z_1 = 0 \textrm{ and } z_i = 2^{i-2} \, \, \forall i=2, 3, \dots.
\end{align}
Note that $\iota_k(s,a) \in \{ 0, 1, 2, 4, 8, 16 \dots \}$ is an integer
indicating the influence of the state-action pair on the value function of
$\pi^k$. Similarly, we define the \emph{knownness}
\begin{align}
    \kappa_k(s,a) \defeq \max \left\{ z_i \, : \, z_i \leq  \frac{n_k(s,a)}{m w_k(s,a)} \right\} \in \{ 0, 1, 2, 4, \dots \}
\end{align}
which indicates how often $(s,a)$ has been observed relative to its importance. The constant $m$ is defined in Algorithm~\ref{alg:fhalg}.
We can now categorize $(s,a)$-pairs into subsets
\begin{align}
    X_{k, \kappa, \iota} \defeq \{(s, a) \in X_k \, : \,  \kappa_k(s, a) = \kappa, \iota_k(s, a) = \iota\} 
    \quad \textrm{and} \quad \bar X_k = \saspace \setminus X_k
\end{align}
where $X_k = \{ (s,a) \in \saspace \, : \, \iota_k(s,a) > 0 \}$ is the
\emph{active set} and $\bar X_k$ the set of state-action pairs that are very
unlikely under the current policy.  Intuitively, the model of \algname is
accurate if only few $(s,a)$ are in categories with low knownness -- that is,
important under the current policy but have not been observed often so far. 
Recall that over time observations are generated under many policies (as 
the policy is recomputed), so this condition does not always hold.  We
will therefore distinguish between phases $k$ where $|X_{k, \kappa, \iota}|
\leq \kappa$ for all $\kappa$ and $\iota$ and phases where this condition is
violated.  The condition essentially allows for only a few $(s,a)$ in
categories that are less known and more and more $(s,a)$ in categories that are
more well known.  In fact, we will show that the policy is $\epsilon$-optimal
with high probability in phases that satisfy this condition.

We first show the validity of the confidence sets $\mathcal M_k$.
\begin{lem}[Capturing the true MDP whp.]
    $M \in \mathcal M_k$ for all $k$ with probability at least $1 - \delta / 2$.
    \label{lem:mdp_capture}
\end{lem}
\vspace{-4mm}
\begin{proof}[Proof Sketch]
    By combining Hoeffding's inequality, Bernstein's inequality and the
    concentration result on empirical standard deviations by \citet{Maurer2009} with the
    union bound, we get that $p(s'|s,a) \in \mathcal P$ with probability at
    least $1 - \delta_1$ for a single phase $k$, fixed $s,a \in \saspace$ and fixed
    $s' \in \succS{s,a}$.  We then show that the number of model updates is
    bounded by $U_{\max}$ and apply the union bound. 
\end{proof}

The following lemma bounds the number of episodes in which $\forall \kappa, \iota : \, |X_{k, \kappa, \iota}| \leq \kappa$ is violated with high probability. 
\begin{lem}
    Let $E$ be the number of episodes $k$ for which there are $\kappa$ and $\iota$ with $|X_{k, \kappa, \iota}| > \kappa$, i.e.
        $E = \sum_{k=1}^\infty \indicator{\exists (\kappa, \iota) \, : \, |X_{k, \kappa, \iota}| > \kappa}$
    and assume that
        $m \geq  \frac{6 H^2}{\epsilon} \ln \frac{2 E_{\max}}{\delta}.$
    Then $\prob (E \leq 6 N E_{\max}) \geq 1 - \delta / 2$ where
$N = \numsa m$ and $E_{\max} =  \log_2 \frac{H}{w_{\min}}  \log_2 \numS$.
\label{lem:unbalanced_episodes_bound}
\end{lem}
\vspace{-4mm}
\begin{proof}[Proof Sketch]
    We first bound the total number of times a fixed pair $(s,a)$ can be
    observed while being in a particular category $X_{k, \kappa, \iota}$ in all
    phases $k$ for $1 \leq \kappa < \numS$.  We then show that
    for a particular $(\kappa,\iota)$, the number of episodes where $|X_{k,
    \kappa, \iota}| > \kappa$ is bounded with high probability, as the value of
    $\iota$ implies a minimum probability of observing each $(s,a)$ pair in
    $X_{k, \kappa, \iota}$ in an episode. Since the observations are not
    independent we use martingale concentration results to 
    show the statement
    for a fixed $(\kappa,\iota)$. The desired result follows with the union bound over all relevant
    $\kappa$ and $\iota$.  
\end{proof}

The next lemma states that in episodes where the condition $\forall \kappa, \iota: \,
|X_{k, \kappa, \iota}| \leq \kappa$ is satisfied and the true MDP is in the
confidence set, the expected optimistic policy value is close to the true value. 
This lemma is the technically most involved part of the proof.
\begin{lem}[Bound mismatch in total reward]
    Assume $M \in \mathcal M_k$. If $| X_{k, \kappa, \iota}| \leq \kappa$ for
    all $(\kappa, \iota)$ and $0 < \epsilon \leq 1$  
    and 
    $
        m \geq 512 \frac{\maxNumSucc H^2}{\epsilon^2}(\log_2 \log_2 H)^2 \log_2^2 \left(\frac{8 H^2 \numS^2}{\epsilon}\right)  \ln \frac{6}{\delta_1}.
    $
    Then $| \tilde V^{\pi^k}_{1:H}(s_0) - V^{\pi^k}_{1:H}(s_0) | \leq \epsilon$.
    \label{lem:balanced_eps_good}
\end{lem}
\vspace{-6mm}
\begin{proof}[Proof Sketch]
Using basic algebraic transformations, we show that $|p -
\tilde p| \leq \sqrt{\tilde p (1 - \tilde p) }  O\left(\sqrt{\frac{1}{n} \ln
\frac 1 {\delta_1}} \right) + O\left(\frac{1}{n} \ln \frac 1 {\delta_1} \right)$ for
each $\tilde p, p \in \mathcal P$ in the confidence set as defined in
Eq.~\ref{eqn:conf_set_bern}. Since we assume $M \in \mathcal M_k$, we
know that  $p(s'|s,a)$ and $\tilde p(s' | s,a)$ satisfy this bound with $n(s,a)$ for all
$s$,$a$ and $s'$.
We use that to bound the difference of the expected value function of the successor state in $M$ and $\tilde M$, proving that  
$    | (P_i - \tilde P_i) \tilde V_{i+1:j}(s)| \leq 
    O\left(\frac{\maxNumSucc H}{n(s, \pi(s))} \ln \frac 1 {\delta_1} \right) + O\left(\sqrt{\frac{C}{n(s, \pi(s))} \ln
\frac 1 {\delta_1}} \right)\tilde \sigma_{i:j}(s)$,
where the local variance of the value function is defined as
$    \sigma_{i:j}^2(s,a) \defeq  \mathbb E\left[(V^\pi_{i+1:j}(s_{i+1}) - P^\pi_i V^\pi_{i+1:j}(s_i))^2 | s_i = s, a_i = a \right] \quad \textrm{and} \quad
\sigma_{i:j}^2(s) \defeq  \sigma_{i:j}^2(s, \pi_i(s)).$
This bound then is applied to $| \tilde V_{1:H}(s_0) - V_{1:H}(s_0)| \leq
\sum_{t=0}^{H-1} P_{1:t} | (P_t - \tilde P_t) \tilde V_{t+1:H}(s)|$.  
The basic idea is to split the bound into a sum of two parts by 
partitioning of the $(s,a)$ space by knownness, e.g. 
that is $(s_t, a_t) \in \bar X_{\kappa, \iota}$ for
all $\kappa$ and $\iota$ and $(s_t, a_t) \in \bar X$.  Using the fact that $w(s_t, a_t)$ and $n(s_t, a_t)$ are tightly
coupled for each $(\kappa, \iota)$, we can bound the expression eventually
by $\epsilon$.  The final key ingredient in the remainder of the proof is to bound $\sum_{t=1}^{H} P_{1:t-1}
\sigma_{t:H}(s)^2$ by $O(H^2)$ instead of the trivial bound $O(H^3)$. To this
end, we show the lemma below.
\end{proof}

\begin{lem}
    The variance of the value function defined as
    $\varV^\pi_{i:j} (s) \defeq  \mathbb E \left[ \left(\sum_{t=i}^j r_t(s_t) - V_{i:j}^\pi(s_i)\right)^2 | s_i = s \right]$ 
    satisfies a Bellman equation
        $\varV_{i:j} = P_{i} \varV_{i+1:j} + \sigma^2_{i:j}$
    which gives
        $\varV_{i:j} = \sum_{t=i}^j P_{i:t-1} \sigma^2_{t:j}$.
    Since $0 \leq \varV_{1:H} \leq H^2 r_{\max}^2$, it follows that
        $0 \leq \sum_{t=1}^j P_{i:t-1} \sigma^2_{t:j}(s) \leq H^2 r_{\max}^2$ for all $s \in \statespace$.
    \label{lem:varV_bellman}
\end{lem}
\begin{proof}[Proof Sketch]
The proof works by induction and uses fact that the value function satisfies
the Bellman equation and the tower-property of conditional expectations. 
\end{proof}

\paragraph{Proof Sketch for Theorem~\ref{thm:upper_bound}.}
The proof of Theorem~\ref{thm:upper_bound} consists of the following major parts:
\vspace{-3mm}
\begin{enumerate}
    \item The true MDP is in the set of  MDPs $\mathcal M_k$ for all phases $k$
        with probability at least $1- \frac{\delta}{2}$ (Lemma~\ref{lem:mdp_capture}).
    \item The \evi algorithm computes a value function whose optimistic
     value is higher than the optimal reward in the true MDP with 
probability at least $1-\delta / 2$ (Lemma~\ref{lem:planning}).
    \item 
        The number of episodes with $|X_{k, \kappa, \iota}| > \kappa$
        for some $\kappa$ and $\iota$ are bounded with probability at least $1 - \delta/2$ by
        $\tilde O(\numsa m)$ if $m = \tilde \Omega\left( \frac{H^2}{\epsilon} \ln \frac{\numS}{\delta} \right)$ (Lemma~\ref{lem:unbalanced_episodes_bound}). 
    \item If $|X_{k, \kappa, \iota}| \leq \kappa$ for all $\kappa$, $\iota$, i.e.,
        relevant state-action pairs are sufficiently known and $m = \tilde \Omega\left(\frac{\maxNumSucc H^2}{\epsilon^2} \ln \frac{1}{\delta_1} \right)$, 
        then the optimistic value computed is $\epsilon$-close to the 
true MDP value. 
Together with part 2, we get that with high probability, the policy $\pi^k$ is $\epsilon$-optimal in this case. 
\item From parts~3 and ~4, with probability $1 - \delta$,
there are at most $\tilde O \left( \frac{C \numsa H^2}{\epsilon^2} \ln
\frac{1}{\delta}\right)$ episodes that are not $\epsilon$-optimal.  
\end{enumerate}

\vspace{-3mm}
\section{Lower PAC Bound}
\vspace{-3mm}

\begin{figure}
\tikzset{every state/.style={minimum size=0pt}}
\centering 
\scalebox{.65}{
\begin{tikzpicture}[node distance=1.1cm,>=stealth',bend angle=60,auto,decoration={
    markings,
    mark=at position 0.5 with {\arrow{>}}}]

    \node [state] (s0) {$0$};
    \node [state, right of=s0, node distance=3cm] (s2) {$2$};
    \node [state, above of=s2] (s1) {$1$};
    \node [below of=s2, node distance=.6cm] (dots) {$\vdots$};
    \node [state, below of=s2, node distance=1.4cm] (sn) {$n$};
    \node [state, right of=s1, node distance=3cm] (sp) {$+$};
    \node [state, right of=sn, node distance=3cm] (sm) {$-$};
    \node [right of=sp, node distance=2cm] (rp) {$r(+) = 1$};
    \node [right of=sm, node distance=2cm] (rm) {$r(-) = 0$};
    \node [above of=s0, node distance=1.2cm, anchor=west] (p0) {$p(i | 0, a) = \frac 1 n$};
    \node [below of=rp, node distance=.7cm, ] (pp) {$p(+ | i, a) = \frac 1 2 + \epsilon'_i(a)$};
    \node [above of=rm, node distance=.7cm, ] (pm) {$p(- | i, a) = \frac 1 2 - \epsilon'_i(a)$};
\path[->] 
(s0) edge  (s1)
(s0) edge  (s2)
(s0) edge  (sn)
(s1) edge  (sp)
(s1) edge  (sm)
(s2) edge  (sp)
(s2) edge  (sm)
(sn) edge  (sp)
(sn) edge   (sm)
(sp) [loop right, looseness=10] edge (sp)
(sm) [loop right, looseness=10] edge (sm)
;
\end{tikzpicture}}
\vspace{-3mm}
\caption{Class of a hard-to-learn finite horizon MDPs.
    The function $\epsilon'$ is defined as $\epsilon'(a_1) = \epsilon / 2$,
$\epsilon'(a_i^*) = \epsilon$ and otherwise $\epsilon'(a) = 0$ where $a_i^*$ is
an unknown action per state $i$ and $\epsilon$ is a
parameter.}
\vspace{-2mm}
    \label{fig:finite_hard_mdp}
\end{figure}
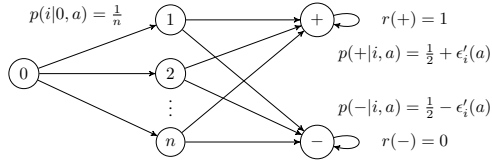

\begin{thm}
    There exist positive constants $c_1$, $c_2$, $\delta_0$, $\epsilon_0$ such
    that for every $\delta \in (0, \delta_0 )$ and $\epsilon
    \in (0, \epsilon_0)$ and for every algorithm A that satisfies a PAC
    guarantee for $(\epsilon, \delta)$ and outputs a deterministic policy,
    there is a fixed-horizon episodic MDP $M_{hard}$ with
    \begin{align}
        \mathbb E[n_A] \geq & \frac{c_1 (H-2)^2 (\numA-1) (\numS
        -3)}{\epsilon^2} \ln \left( \frac{c_2}{\delta + c_3 }\right) 
        =  \Omega\left( \frac{|\saspace| H^2}{\epsilon^2} \ln \left(\frac{c_2}{\delta + c_3} \right)\right)
        \label{eqn:lower_bound}
    \end{align}
    where $n_A$ is the number of episodes until the algorithm's policy is
    $(\epsilon, \delta)$-accurate. The constants can be set to $ \delta_0 =\frac{ e^{-4}}{80} \approx \frac{1}{5000}$, $\epsilon_0 = \frac{H-2}{640 e^4} \approx H/35000$,
$c_2 =4$ and $c_3 = e^{-4}/80$.
\label{thm:lower_bound}
\end{thm}
The ranges of possible $\delta$ and $\epsilon$ are of similar order than in other state-of-the-art lower bounds for multi-armed bandits
\citep{Mannor2004} and discounted MDPs \citep{Strehl2009,
Lattimore2012}.
They are mostly determined by the bandit result by \citet{Mannor2004} we build on. Increasing the parameter limits $\delta_0$ and $\epsilon_0$ for bandits would immediately result in larger ranges in our lower bound,
but this was not the focus of our analysis.
\vspace{-2mm}
\begin{proof}[Proof Sketch]
    The basic idea is to show that the class of MDPs shown in
    Figure~\ref{fig:finite_hard_mdp} require at least a number of observed
    episodes of the order of Equation~\eqref{eqn:lower_bound}.  From the start
    state $0$, the agent ends up in states $1$ to $n$ with equal probability,
    independent of the action. From each such state $i$, the agent transitions
    to either a good state $+$ with reward $1$ or a bad state $-$ with reward
    $0$ and stays there for the rest of the episode. Therefore, each state
    $i=1, \dots, n$ is essentially a multi-armed bandit with binary rewards of
    either $0$ or $H-2$. For each bandit, the probability of ending up in $+$
    or $-$ is equal except for the first action $a_1$ with $p(s_{t+1} = +| s_t
    = i, a_t = a_1) = 1/2 + \epsilon / 2$ and possibly an unknown optimal action
    $a_i^*$ (different for each state $i$) with $p(s_{t+1} = +| s_t = i, a_t =
    a_i^*) = 1/2 + \epsilon$. 

In the episodic fixed-horizon setting we
    are considering, taking a suboptimal action in one of the bandits does not
    necessarily yield a suboptimal episode. We have to consider the average
    over all bandits instead. In an $\epsilon$-optimal episode, the agent therefore 
    needs to follow a policy that would solve at least a certain portion of all $n$ multi-armed
    bandits with probability at least $1 - \delta$. We show that
    the best strategy for the agent to achieve this is to try to solve all bandits
    with equal probability. The number of samples required to do so then
    results in the lower bound in Equation~\eqref{eqn:lower_bound}. 
\end{proof}
\vspace{-3mm}
Similar MDPs that essentially solve multiple of such multi-armed bandits
    have been used to prove lower sample-complexity bounds for discounted MDPs \citep{Strehl2009,Lattimore2012}. However, the analysis in the
    infinite horizon case as well as for the sliding-window fixed-horizon
    optimality criterion considered by \citet{Kakade2003} is significantly
    simpler. For these criteria, every time step the agent follows a policy that
    is not $\epsilon$-optimal counts as a "mistake". Therefore, every time the
    agent does not pick the optimal arm in any of the multi-armed bandits
    counts as a mistake. This contrasts with our fixed-horizon setting where we must instead consider taking an average over all bandits. 

    \vspace*{-3mm}
\section{Related Work on Fixed-Horizon Sample Complexity Bounds}
\label{sec:literature}
\vspace*{-2mm}
We are not aware of any lower sample complexity bounds beyond multi-armed bandit results that directly apply to our setting. 
Our upper bound in Theorem~\ref{thm:upper_bound} improves upon existing results by at least a factor of $H$. We briefly review those existing results in the following.
\vspace{-3mm}
\paragraph{Timestep bounds.} \citet[Chapter~8]{Kakade2003} proves upper and
lower PAC bounds for a similar setting where the agent interacts indefinitely
with the environment but the interactions are divided in segments of equal
length and the agent is evaluated by the expected sum of rewards until the end
of each segment. The bound
states that there are not more than $\tilde O\left(\frac{\numS^2 \numA
H^6}{\epsilon^3} \ln \frac{1}{\delta}\right)$\footnote{For comparison we adapt existing bounds to our setting. While the original bound stated by \citet{Kakade2003} only has
    $H^3$, an additional $H^3$ comes in through $\epsilon^{-3}$ due to
different normalization of rewards.} 
time steps in which the agents acts $\epsilon$-suboptimal. 
\citet{Strehl2006}
improves the state-dependency of these bounds for their delayed Q-learning algorithm to $\tilde
O\left(\frac{\numS \numA H^5}{\epsilon^4} \ln \frac 1 \delta \right)$.  
However, in episodic MDP it is more natural to consider performance on the
entire episode since suboptimality near the end of the episode is no issue as long as
the total reward on the entire episode is sufficiently high.
\citet{Kolter2009a} use an interesting sliding-window criterion, but prove bounds for a Bayesian setting instead of PAC.
Timestep-based bounds can be applied to the episodic case by
augmenting the original statespace with a time-index per episode to allow
resets after $H$ steps. This adds $H$ dependencies
for each $\numS$ in the original bound which results in a horizon-dependency of at least $H^6$ of
these existing bounds. 
Translating the regret bounds of \texttt{UCRL2} in Corollary~3 by \citet{Jaksch2010} yields
a PAC-bound on the number of episodes of at least $\tilde O\left(\frac{\numS^2 \numA H^3}{\epsilon^2} \ln \frac 1
\delta \right)$ even if one ignores the reset after $H$ time steps.
Timestep-based lower PAC-bounds cannot be applied
directly to the episodic reward criterion.
\vspace{-3mm}
\paragraph{Episode bounds.}
Similar to us, \citet{Fiechter1994} uses the value of
initial states as optimality-criterion, but defines the value w.r.t. the $\gamma$-discounted infinite horizon. His results
of order $\tilde O\left(\frac{\numS^2 \numA
H^7}{\epsilon^2} \ln \frac{1}{\delta}\right)$ episodes of length $\tilde O(1 /
(1 - \gamma)) \approx \tilde O(H)$
are therefore not directly applicable to our setting. 
\citet{AuerOrtner2005} investigate the same setting as we and propose
a UCB-type algorithm that has no-regret, which translates into a basic PAC bound of order $\tilde
O\left(\frac{\numS^{10} \numA H^7}{\epsilon^3} \ln \frac{1}{\delta} \right)$
episodes. We improve on this bound substantially in terms of its dependency on
$H$, $\numS$ and $\epsilon$.
\citet{Reveliotis2007} also consider the episodic undiscounted fixed-horizon
setting and present an efficient algorithm in cases where the transition graph is acyclic 
and the agent knows for each state a policy that
visits this state with a known minimum probability $q$. These assumptions are quite limiting and rarely hold in practice and their bound
of order $\tilde O\left( \frac{\numS \numA H^4 }{\epsilon^2 q} \ln
\frac{1}{\delta} \right)$ explicitly depends on $1/q$.

\vspace{-3mm}
\section{Conclusion}
\vspace{-3mm}
We have shown upper and lower bounds on the sample complexity of episodic
fixed-horizon RL that are tight up to log-factors in the
time horizon $H$, the accuracy $\epsilon$, the number of actions $\numA$ and up
to an additive constant in the failure probability $\delta$. These bounds
improve upon existing results by a factor of at least $H$.  One might hope to
reduce the dependency of the upper bound on $\numS$ to be linear by an analysis
similar to \texttt{Mormax} \citep{Szita2010} for discounted MDPs which has sample
complexity linear in $\numS$ at the penalty of additional dependencies on $H$. 
Our proposed \algname algorithm that achieves our PAC bound can be
applied to directly to a wide range of fixed-horizon episodic MDPs with known
rewards and does not require additional structure such as sparse or acyclic
state transitions assumed in previous work. The empirical evaluation of \algname is an interesting direction for future work.
\vspace{-2mm}
\paragraph{Acknowledgments:} We thank Tor Lattimore for the helpful suggestions and
comments. We are also grateful to Shiau Hong Lim and Ian Osband for discovering small bugs in previous versions of this paper. This work was supported by an NSF CAREER award and the ONR Young
Investigator Program.
\bibliographystyle{unsrtnat-nourl}
\bibliography{mendeley}

\begin{thebibliography}{21}
\providecommand{\natexlab}[1]{#1}
\providecommand{\url}[1]{\texttt{#1}}
\expandafter\ifx\csname urlstyle\endcsname\relax
  \providecommand{\doi}[1]{doi: #1}\else
  \providecommand{\doi}{doi: \begingroup \urlstyle{rm}\Url}\fi

\bibitem[Strehl et~al.(2006{\natexlab{a}})Strehl, Li, Wiewiora, Langford, and
  Littman]{Strehl2006}
Alexander~L. Strehl, Lihong Li, Eric Wiewiora, John Langford, and Michael~L.
  Littman.
\newblock {PAC Model-Free Reinforcement Learning}.
\newblock In \emph{International Conference on Machine Learning},
  2006{\natexlab{a}}.

\bibitem[Kearns and Singh(1999)]{Kearns1999}
Michael~J Kearns and Satinder~P Singh.
\newblock {Finite-Sample Convergence Rates for Q-Learning and Indirect
  Algorithms}.
\newblock In \emph{Advances in Neural Information Processing Systems}, 1999.

\bibitem[Brafman and Tennenholtz(2002)]{Brafman2003}
Ronen~I Brafman and Moshe Tennenholtz.
\newblock {R-MAX -- A General Polynomail Time Algorithm for Near-Optimal
  Reinforcement Learning}.
\newblock \emph{Journal of Machine Learning Research}, 3:\penalty0 213--231,
  2002.

\bibitem[Kakade(2003)]{Kakade2003}
Sham~M. Kakade.
\newblock \emph{{On the Sample Complexity of Reinforcement Learning}}.
\newblock PhD thesis, University College London, 2003.

\bibitem[Auer and Ortner(2005)]{AuerOrtner2005}
Peter Auer and Ronald Ortner.
\newblock {Online Regret Bounds for a New Reinforcement Learning Algorithm}.
\newblock In \emph{Proceedings 1st Austrian Cognitive Vision Workshop}, 2005.

\bibitem[Lattimore and Hutter(2012)]{Lattimore2012}
Tor Lattimore and Marcus Hutter.
\newblock {PAC bounds for discounted MDPs}.
\newblock In \emph{International Conference on Algorithmic Learning Theory},
  2012.

\bibitem[Szita and Szepesv{\'{a}}ri(2010)]{Szita2010}
Istv{\`{a}}n Szita and Csaba Szepesv{\'{a}}ri.
\newblock {Model-based reinforcement learning with nearly tight exploration
  complexity bounds}.
\newblock In \emph{International Conference on Machine Learning}, 2010.

\bibitem[Azar et~al.(2012)Azar, Munos, and Kappen]{Azar2012}
Mohammad~Gheshlaghi Azar, R{\'{e}}mi Munos, and Hilbert~J. Kappen.
\newblock {On the Sample Complexity of Reinforcement Learning with a Generative
  Model}.
\newblock In \emph{International Conference on Machine Learning}, 2012.

\bibitem[Kolter and Ng(2009)]{Kolter2009a}
J~Zico Kolter and Andrew~Y Ng.
\newblock {Near-Bayesian exploration in polynomial time}.
\newblock In \emph{International Conference on Machine Learning}, 2009.

\bibitem[Fiechter(1994)]{Fiechter1994}
Claude-Nicolas Fiechter.
\newblock {Efficient reinforcement learning}.
\newblock In \emph{Conference on Learning Theory}, 1994.

\bibitem[Fiechter(1997)]{Fiechter1997}
Claude-Nicolas Fiechter.
\newblock {Expected Mistake Bound Model for On-Line Reinforcement Learning}.
\newblock In \emph{International Conference on Machine Learning}, 1997.

\bibitem[Reveliotis and Bountourelis(2007)]{Reveliotis2007}
Spyros Reveliotis and Theologos Bountourelis.
\newblock {Efficient PAC learning for episodic tasks with acyclic state
  spaces}.
\newblock \emph{Discrete Event Dynamic Systems: Theory and Applications},
  17\penalty0 (3):\penalty0 307--327, 2007.

\bibitem[Strehl et~al.(2006{\natexlab{b}})Strehl, Li, and Littman]{Strehl2006a}
Alexander~L Strehl, Lihong Li, and Michael~L Littman.
\newblock {Incremental Model-based Learners With Formal Learning-Time
  Guarantees}.
\newblock In \emph{Conference on Uncertainty in Artificial Intelligence},
  2006{\natexlab{b}}.

\bibitem[Strehl et~al.(2009)Strehl, Li, and Littman]{Strehl2009}
Alexander~L Strehl, Lihong Li, and Michael~L Littman.
\newblock {Reinforcement Learning in Finite MDPs : PAC Analysis}.
\newblock \emph{Journal of Machine Learning Research}, 10:\penalty0 2413--2444,
  2009.

\bibitem[Jaksch et~al.(2010{\natexlab{a}})Jaksch, Ortner, and Auer]{Auer2009}
Thomas Jaksch, Ronald Ortner, and Peter Auer.
\newblock {Near-optimal Regret Bounds for Reinforcement Learning}.
\newblock In \emph{Advances in Neural Information Processing Systems},
  2010{\natexlab{a}}.

\bibitem[Strehl and Littman(2008)]{Strehl2008}
Alexander~L. Strehl and Michael~L. Littman.
\newblock {An analysis of model-based Interval Estimation for Markov Decision
  Processes}.
\newblock \emph{Journal of Computer and System Sciences}, 74\penalty0
  (8):\penalty0 1309--1331, dec 2008.

\bibitem[Sobel(1982)]{Sobel1982}
Matthew~J Sobel.
\newblock {The Variance of Markov Decision Processes}.
\newblock \emph{Journal of Applied Probability}, 19\penalty0 (4):\penalty0
  794--802, 1982.

\bibitem[Maurer and Pontil(2009)]{Maurer2009}
Andreas Maurer and Massimiliano Pontil.
\newblock {Empirical Bernstein Bounds and Sample-Variance Penalization.}
\newblock In \emph{Conference on Learning Theory}, 2009.

\bibitem[Mannor and Tsitsiklis(2004)]{Mannor2004}
Shie Mannor and John~N Tsitsiklis.
\newblock {The Sample Complexity of Exploration in the Multi-Armed Bandit
  Problem}.
\newblock \emph{Journal of Machine Learning Research}, 5:\penalty0 623--648,
  2004.

\bibitem[Jaksch et~al.(2010{\natexlab{b}})Jaksch, Ortner, and Auer]{Jaksch2010}
Thomas Jaksch, Ronald Ortner, and Peter Auer.
\newblock {Near-optimal Regret Bounds for Reinforcement Learning}.
\newblock \emph{Journal of Machine Learning Research}, 11:\penalty0 1563--1600,
  2010{\natexlab{b}}.

\bibitem[Chung and Lu(2006)]{Chung2006}
Fan Chung and Linyuan Lu.
\newblock {Concentration Inequalities and Martingale Inequalities: A Survey}.
\newblock \emph{Internet Mathematics}, 3\penalty0 (1):\penalty0 79--127, 2006.

\end{thebibliography}

\newpage 

\appendix
\counterwithin{lem}{section}
\appendixpage
\startcontents[sections]
\printcontents[sections]{l}{1}{\setcounter{tocdepth}{2}}
\newpage
\section{Fixed-Horizon Extended Value Iteration}
\label{sec:fhevi}

We want to find a policy $\pi^k$ and optimistic ${\tilde M}_k \in
\mathcal M_k$ which have the highest total reward $R^{\pi^k}_{{\tilde M}_k} = \max_{\pi, M' \in \mathcal M_k} R^{\pi}_{M'}$.
Note that $\pi^k$ is an optimal policy for $M_k$ but not necessarily for $M$.
To facilitate planning, we relax this problem and instead compute a policy and optimistic MDP with $R^{\pi^k}_{{\tilde M}_k} = \max_{\pi, M' \in \mathcal M'_k} R^{\pi}_{M'}$
with
\begin{align}
    \mathcal M'_k \defeq \big\{ & \tilde{M} \in \mathcal M_{\textrm{nonst.}} \, : \, \forall (s,a) \in \saspace, t =1\dots H, s' \in \succS{s,a}\\
        &{\tilde p}_{t}(s'|s,a) \in \operatorname{conv}(\confset{${\hat p}(s' |s,a), n(s, a)$}) \big\}.
        \label{eqn:approx_mdp_set}
\end{align}
We only require the transition probabilities to be in the convex hull of the confidence sets instead of the confidence sets. Since this is a relaxation, we have $\mathcal M_k \subseteq \mathcal M'_k$.
We can find such a policy by dynamic programming similar to
extended value iteration \citep{Strehl2008,AuerOrtner2005}. The optimal Q-function can be
computed as 
$\tilde Q_{H:H}^*(s,a) =  r_H(s)$
and for $i=H-1, \dots,2,  1$ as
\begin{align}
    \tilde Q_{i:H}^*(s,a) = & r_i(s) + \max_{\tilde p_i \in \mathcal P_{s,a}}\left\{\sum_{s' \in \succS{s, a}} \tilde p_i \max_{b \in \actionspace} \tilde Q_{i+1:H}^*(s', b) \right\}
\end{align}
The feasible set is defined as $\mathcal P_{s,a} \defeq \{ p \in
[0,1]^{\numSucc{s,a}} \, | \,   \| p \|_1 = 1,  \forall s' \in \succS{s,a}: \, p(s') \in \operatorname{conv}(\confset(\hat
p(s' | s,a), n(s,a))) \}$. 
The optimal policy $\pi^k_t(s)$ at time $t$ is then simply the maximizer of
the inner $\max$ operator and the transition probability $\tilde p_t(\cdot
| s,a)$ is the maximizer of the outer maximum. The inner $\max$ can be solved
efficiently by enumeration and the outer maximum similar to extended value
iteration \citep{Strehl2008}. 
The basic idea is to put as much
probability mass as possible to successor states with highest value. See the following algorithm for the implementation details.
\begin{function}[H]
    $\tilde Q^*_{H:H}(s, a) = r_H(s) \quad \forall s,a \in \saspace$ \tcp*[r]{$O(\numS \numA)$}
    \For(\tcp*[f]{$O(H \numS \log \numS + H \numS \numA \maxNumSucc))$}){$t=H-1$ \KwTo $1$} 
    {
        $\pi_{t+1}(s) := \argmax_{a \in \actionspace} \tilde Q^*_{t+1:H}(s, a) \quad \forall s \in \statespace$
        \tcp*[r]{$O(\numS \numA)$}
        sort states $s^{(1)}, \dots s^{(\numS)}$ such that \\
        $\qquad\tilde Q^*_{t+1:H}(s^{(i)}, \pi_{t+1}(s^{(i)})) \geq \tilde Q^*_{t+1:H}(s^{(i+1)}, \pi_{t+1}(s^{(i+1)}))$
        \tcp*[r]{$O(\numS \log \numS)$}
        \For(\tcp*[f]{$O(\numS \numA \maxNumSucc)$}){$s,a \in \saspace$}{
            $\tilde p_t(s' | s,a) :=  \min \confset(\hat p(s' | s,a), n(s,a)) \quad \forall s' \in \succS{s,a}$
            \tcp*[r]{$O(\maxNumSucc)$}
            $\Delta := 1 - \sum_{s' \in \succS{s,a}} \tilde p_t(s' | s,a)$
            \tcp*[r]{$O(\maxNumSucc)$}
            $i := 1$ 
            \tcp*[r]{$O(1)$}
            \While(\tcp*[f]{$O(\maxNumSucc)$}){$\Delta > 0$}{
                $s' := s^{(i)}$\;
                $\Delta' := \min\{ \Delta,  \max \confset(\hat p(s' | s,a), n(s,a)) - \tilde p_t(s' | s,a)\}$\;
                $\tilde p_t(s' | s,a) := \tilde p_t(s' | s,a) + \Delta'$\;
                $\Delta := \Delta - \Delta'; i := i + 1$\;
            }
            $\tilde Q^*_{t:H}(s,a) = \sum_{s' \in \succS{s,a}} \tilde p_t(s' | s, a) \tilde Q_{t+1:H}^*(s', \pi_{t+1}(s'))$
            \tcp*[r]{$O(\maxNumSucc)$}
        }
    }
    $\pi_{1}(s) := \argmax_{a \in \actionspace} \tilde Q^*_{1:H}(s, a) \quad \forall s \in \statespace$ 
    \tcp*{$O(\numS \numA$}
    \KwRet{MDP with transition probabilities $\tilde p_t$, optimal policy $\pi$}
    \caption{FixedHorizonEVI ($\mathcal M$)}
\label{alg:evi}
\end{function}

Note that due to
the nonlinear constraint in Equation~\eqref{eqn:conf_set_var},
$\confset(\hat p(s' | s,a), n(s,a))$ may be the union of two disjoint intervals
instead of one interval. Still, $\min$- and $\max$-operations on the confidence
sets can be computed readily in constant time. Therefore, the transition
probabilities $\tilde p_t( \cdot | s, a)$ for a single time step $t$ and
state-action pair $s,a$ can be computed in $O(\numS \numA \maxNumSucc)$ given sorted
states. Sorting the states takes $O(\numS \log \numS)$ which results in $O(H
\numS \log \numS + H \numS \numA \maxNumSucc)$ runtime complexity of \evi (see comments
in Function~\ref{alg:evi}).
The Algorithm requires $O(H \numS \numA \maxNumSucc)$ additional space besides the
storage requirements of the input MDP $\mathcal M$ as the transition
probabilities $\tilde p_t$ are returned by the algorithm. If those are not
required and only the optimal policy is of interest, the additional space can
be reduced to $O(\numS \numA)$. 

\begin{lem}[Validity of optimistic planning]
    \evi{$\mathcal M_k$} returns $\tilde M, \pi^k = \argmax_{M \in \mathcal M'_k, \pi} R^\pi_M$. Since $\mathcal M_k \subseteq \mathcal M'_k$, it also holds that $R^{\pi^k}_{\tilde M} \geq \max_{M \in \mathcal M_k, \pi} R^\pi_M$.
    \label{lem:planning}
\end{lem}
\begin{proof}[Proof Sketch]
    This result can be proved straight-forwardly by showing that $\pi^k$ is
    optimal in the last time step $H$ with highest possible reward and then
    subsequently for all previous time steps inductively. It follows
    directly from the definition of the algorithm in Function~\ref{alg:evi}
    that the returned MDP is in $\mathcal M_k'$.
\end{proof}

\section{Runtime- and Space-Complexity of \algname}
Sampling one episode and updating the respective $v$ variables has $O(H)$
runtime. Theorem~\ref{thm:upper_bound} states that after at most 
$\tilde
O\left( \frac{H^2 \maxNumSucc| \saspace|}{\epsilon^2} \ln
\frac{1}{\delta}\right)$ 
observed episodes, the current policy is
$\epsilon$-optimal with sufficiently high probability. This results in a total
runtime for sampling of 
$\tilde
O\left( \frac{H^3 \maxNumSucc| \saspace|}{\epsilon^2} \ln
\frac{1}{\delta}\right)$. 

Each update of the policy involves updating the $n$ variables and $\mathcal M_k$ which takes runtime $O(\maxNumSucc)$ and a call of \evi with runtime cost $O(H \numS \numA \maxNumSucc + H \numS \log \numS)$. 
From Lemma~\ref{lem:num_total_updates} below, we know that the policy can be updated at most $\maxUpdates$ times which a gives total runtime for policy updates of 
\begin{align}
    O(\maxUpdates H \numS (\numA \maxNumSucc + \log \numS)) 
    =& O \left(H \numS^2 \numA (\numA \maxNumSucc + \log \numS) \log \frac{\numS^2 H^2}{\epsilon}\right)\\
    =& \tilde O \left(H \numS^2 \numA^2 \maxNumSucc \log \frac{1}{ \epsilon} \right).
\end{align}
The total runtime of \algname before the policy is $\epsilon$-optimal with probability at least $1- \delta$ is therefore
\begin{align}
    \tilde O \left(\frac{H^3 \numS^2 \numA^2 \maxNumSucc}{\epsilon^2}  \ln \frac{1}{ \delta} \right).
\end{align}

The space complexity of \algname is dominated by the requirement to store statistics for each possible transition which gives $O(\numS \numA \maxNumSucc)$ complexity.

\section{Detailed Proofs for the Upper PAC Bound}

\subsection{Bound on the Number of Policy Changes of \algname}

\begin{lem}
    The total number of updates is bounded by $\maxUpdates = |\mathcal S \times
    \mathcal A| \log_2 \frac{| \mathcal S | H}{w_{\min}}$.
    \label{lem:num_total_updates}
\end{lem}
\begin{proof}
    First note that $n(s,a)$ is never never decreasing and no updates happen
    once $n(s,a) \geq \numS m H$ for all $(s,a)$.  In each update, the $n(s,a)$
    of exactly one $(s,a)$ pair increases by $\max\{ m w_{\min}, n(s,a)\}$. For a
    single $(s,a)$ pair, such updates can happen only $\log_2(\numS m H) -
    \log_2(m w_{\min})$ times. Hence, there are at most $| \statespace \times
    \mathcal A| \log_2 \frac{\numS m H}{w_{\min} m}$ updates in total.
\end{proof}

\subsection{Proof of Lemma~\ref{lem:mdp_capture} -- Capturing the true MDP}
\begin{proof}
    For a single $(s,a)$ pair, $s' \in \succS{s,a}$ and $k$, we can treat the event
    that $s'$ is the successor state of $s$ when chosing action $a$ as a Bernoulli
    random variable with probability $p(s' | s,a)$. Using Hoeffding's
    inequality,\footnote{While the considered random variables are strictly
    speaking not necessarily independent, they can be treated as such for the
    concentration inequalities applied here. See Appendix~A of \citet{Strehl2008}
    for details.} we then realize that
    \[
        |p(s' | s, a) - \hat p(s' | s,a) | \leq \sqrt{\frac{\ln(6 / \delta_1)}{2n}}
    \]
    and by Bernstein's inequality
    \begin{align}
        |p(s' | s, a) - \hat p(s' | s,a) | \leq \sqrt{\frac{2 p(s' | s,a) (1 - p(s' | s,a)) \ln(6 / \delta_1)}{n}} + \frac{1}{3n} \ln(6 / \delta_1)
    \end{align} 
    with probability at least $1 - \delta_1 / 3$ respectively.
    Using both inequalities of Theorem~10 by \citet{Maurer2009}\footnote{The empirical variance denoted by $V_n(\mathbf X)$ by \citet{Maurer2009} is $\tilde p(s' | s,a) (1 - \tilde p(s' | s,a))$ in our case and $\mathbb E V_n$ is the true variance which amounts to $p(s' | s,a) (1 - p(s' | s,a))$ for us.}, we have
    \begin{align}
        |\sqrt{p(s' | s,a) (1 - p(s' | s,a))} - \sqrt{\hat p(s' | s,a) (1 - \hat p(s' | s,a))}|  \leq \sqrt{\frac{2 \ln(6 / \delta_1)}{n-1}}
        \label{eqn:maurer_std_bound}
    \end{align}
    for $n > 1$ with probability at least $1 - \delta_1 / 3$. All three inequalities hold with probability $1 - \delta_1$ by the union bound.
    Applying Inequality~\eqref{eqn:maurer_std_bound} to Bernstein's inequality, we obtain
    \begin{align}
        |p(s' | s, a) - \hat p(s' | s,a) | 
        \leq & \sqrt{\frac{2 p(s' | s,a) (1 - p(s' | s,a)) \ln(6 / \delta_1)}{n}} + \frac{1}{3n} \ln(6 / \delta_1)\\
        \leq & \left( \sqrt{\hat p(s' | s,a) (1 - \hat p(s' | s,a))} +  \sqrt{\frac{2 \ln(6 / \delta_1)}{n-1}} \right)\sqrt{ \frac{2  \ln(6 / \delta_1)}{n}} + \frac{1}{3n} \ln(6 / \delta_1)\\
        \leq & \sqrt{ \frac{2 \hat p(s' | s,a) (1 - \hat p(s' | s,a)) \ln(6 / \delta_1)}{n}} + \frac{7}{3(n-1)} \ln(6 / \delta_1).
    \end{align}
    By Lemma~\ref{lem:num_total_updates}, there are at most $\maxUpdates$ updates
    and so there are at most $\maxUpdates$ different $k$ to consider.  Since in
    each update, only a single $(s,a)$ pair with at most $\maxNumSucc$ successor states is updated, for all $k$ and
    $(s,a)$, there are only $\maxUpdates \maxNumSucc$ different ${\hat p}(s' | s,a)$ to consider. Applying the union bound, we
    get that
    $M \notin \mathcal M_k$ for any $k$ with probability at most
    $\maxUpdates \maxNumSucc \delta_1$.
    By setting $\delta_1 = \frac{\delta}{ 2 \maxNumSucc \maxUpdates}$ we get the desired result.
\end{proof}

\subsection{Bounding the number of episodes with $\kappa > |X_{k, \kappa, \iota}|$ for some $\kappa, \iota$}
\label{sec:proof_unbalanced_eps_bound}
Before presenting the proof of Lemma~\ref{lem:unbalanced_episodes_bound} which bounds the total number of episodes where there is a $\kappa$ and $\iota$ such that $\kappa > |X_{k, \kappa, \iota}|$, we establish a bound for each individual $\kappa$ and $\iota$ in the following two additional lemmas.
\begin{lem}[Bound on observations of $X_{\cdot, \kappa, \iota}$]
    The total number of observations of $(s,a) \in X_{k, \kappa, \iota}$ where $\kappa \in [1, \numS -1]$ and $\iota > 0$ over
    all phases $k$ is at most $3 | \saspace|  m w_\iota \kappa$. The variable $w_\iota$ is the smallest possible weight of a $(s,a)$-pair that has importance $\iota$.
    \label{lem:ki_obs_bound}
\end{lem}
\begin{proof}
    We denote the smallest possible weight for any $(s,a)$ pair such that $\iota(s,a) = \iota$ by
    $w_\iota \defeq \min\{ w(s,a) : \iota_k(s,a) = \iota \}$. Note that $w_{\iota+1} = 2 w_\iota$ for $\iota > 0$. 
    Consider any phase $k$ and fix $(s,a) \in X_{k, \kappa, \iota}$ with $\iota > 0$. 
    Since we assumed
    $\iota_k(s,a) = \iota > 0$, we have
    $w_\iota \leq w_k(s,a) <  2w_{\iota}$.
     From $\kappa_k(s,a) = \kappa$, it follows that 
     \begin{align}
        \frac{n_k(s,a)}{2 m w_k(s,a)} \leq \kappa \leq \frac{n_k(s,a)}{m w_k(s,a)}
    \end{align}
    which implies that
    \begin{align}
        m w_{\iota}\kappa  \leq m w_k(s,a) \kappa \leq n_k(s,a) \leq 2 m w_k(s,a) \kappa \leq 4 m w_{\iota} \kappa.
        \label{eqn:ki_n}
    \end{align}
    Hence, each state can only be observed $3 m w_\iota$ times while being in
    $\{(s,a) \in X_{k, \kappa, \iota}\,: \, k \in \mathbb N \}$.
\end{proof}

\begin{lem}
    The number of episodes $E_{\kappa, \iota}$ in phases with $|X_{k, \kappa,
    \iota}| > \kappa$ is bounded for every $\alpha \geq 3$ with high probability,
    \[
        P(E_{\kappa, \iota} > \alpha N) \leq \exp\left( - \frac{\beta w_\iota (\kappa + 1) N}{H} \right)
    \]
    where $N = |\saspace|m$ and $\beta =  \frac{\alpha (3 / \alpha - 1)^2}{7/3 - 1 / \alpha}$.
    \label{lem:bound_ki_epis}
\end{lem}
\begin{proof}
    Let $\nu_i \defeq \sum_{t=1}^H \indicator{(s_t,a_t) \in X_{k, \kappa, \iota}}$ be
    the number of observations of $(s,a)$ in $X_{k,\kappa, \iota}$ in the
    $i$th epsiode with $X_{k, \kappa, \iota} > \kappa$. We have $i \in \{1, \dots
    E_{\kappa, \iota}\}$) and $k$ is the phase that episode $i$ belongs to. 

    Since $X_{k, \kappa, \iota} \geq \kappa + 1$ and all states in
    partition $(\kappa, \iota)$ have $w_k(s,a) \geq w_\iota$ , we get 
    \begin{equation}
        \mathbb E[ \nu_i |
    \nu_1, \dots \nu_{i-1}] \geq (\kappa + 1) w_\iota.
    \label{eqn:exp_nu_lb}
    \end{equation}
    Also $\Var[\nu_i | \nu_1 \dots \nu_{i-1}] \leq \mathbb E[\nu_i | \nu_1,
    \dots \nu_{i-1}] H$ as $\nu_i \in [0, H]$.

    To reason about $E_{\kappa, \iota}$, we define the continuation
    \[
        \nu_i^+ \defeq \begin{cases} \nu_i & \textrm{if }i \leq E_{\kappa, \iota}\\
        w_\iota (\kappa + 1) & \textrm{otherwise}
                            \end{cases}
    \]
    and the centered auxiliary sequence
    \[
        \bar \nu_i \defeq \frac{\nu_i^+ w_\iota (\kappa + 1)}{\mathbb E[\nu_i^+ | \nu_1^+, \dots \nu_{i-1}^+]}.
    \]
    By construction 
    \[
        \mathbb E[\bar \nu_i | \bar \nu_1, \dots \bar \nu_{i-1}] = 
        w_\iota (\kappa + 1) \frac{\mathbb E[\nu_i^+ | \bar \nu_1, \dots, \bar \nu_{i-1}]}{ \mathbb E[\nu_i^+ | \nu_1^+, \dots \nu_{i-1}^+]} = 
        w_\iota (\kappa + 1).
    \]
    By Lemma~\ref{lem:ki_obs_bound}, we have that $E_{\kappa, \iota} > \alpha N$ only if
    \[
        \sum_{i=1}^{\alpha N} \bar \nu_i \leq 3 N w_\iota \kappa \leq 3 N w_\iota (\kappa + 1).
    \]
    Define now the martingale 
    \begin{equation}
        B_i \defeq \mathbb E\left[ \sum_{j=1}^{\alpha N} \bar \nu_j | \bar \nu_1, \dots \bar \nu_i\right]
        = \sum_{j=1}^{i} \bar \nu_j + \sum_{j=i+1}^{\alpha N} \mathbb E[ \bar \nu_j | \bar \nu_1 \dots \bar \nu_{i}] 
    \end{equation}
which gives $B_0 = \alpha N w_\iota (\kappa + 1)$ and $B_{\alpha N} =
\sum_{i=1}^{\alpha N} \bar \nu_i$. Further, since $\nu_i^+ \in [0, H]$ and Equation~\eqref{eqn:exp_nu_lb}, we have
\begin{align}
    |B_{i+1} - B_i| &= | \bar \nu_i - \mathbb E[ \bar \nu_i | \bar \nu_1, \dots, \bar \nu_{i-1}]|
= \left| \frac{w_\iota (\kappa + 1) (\nu_i^+ - \mathbb E[ \nu_i^+ | \bar \nu_1, \dots \bar \nu_{i-1}])}
{\mathbb E[\nu_i^+ | \nu_1^+, \dots \nu_{i-1}^+]} \right| \\
& \leq \left| \nu_i^+ - \mathbb E[ \nu_i^+ | \bar \nu_1, \dots \bar \nu_{i-1}] \right| \leq H.
\end{align}
Using 
\begin{align}
    \sigma^2 \defeq& \sum_{i=1}^{\alpha N} \Var[B_i - B_{i-1} | B_1 - B_0, \dots B_{i-1} - B_{i-2}] \\
    =& \sum_{i=1}^{\alpha N} \Var[\bar \nu_i | \bar \nu_1, \dots \bar \nu_{i-1}] 
    \leq  \alpha N H w_\iota (\kappa + 1) = H B_0
\end{align}
we can apply Theorem~22 by \citet{Chung2006} and obtain
\begin{align}
    \prob(E_{\kappa, \iota} > \alpha N) 
    & \leq \prob \left( \sum_{i=1}^{\alpha N} \bar \nu_i \leq 3 N w_\iota (\kappa + 1) \right)\\
    & = \prob ( B_{\alpha N} - B_0 \leq 3 B_0 / \alpha -  B_0)
    = \prob ( B_{\alpha N} - B_0 \leq  - \left(1 - 3 / \alpha\right) B_0)\\
    & \leq  \exp \left(- \frac{(3 / \alpha - 1)^2 B_0^2}{2 \sigma^2 + H (1/3 - 1 / \alpha) B_0} \right)
\end{align}
for $\alpha \geq 3$. We can further simplify the bound to
\begin{align}
    \prob(E_{\kappa, \iota} > \alpha N) 
    & \leq  \exp \left(- \frac{(3 / \alpha - 1)^2 B_0^2}{2 H B_0 + H (1/3 - 1 / \alpha) B_0} \right)\\
    & \leq  \exp \left(- \frac{(3 / \alpha - 1)^2}{2 + (- 1/ \alpha + 1/3)} \frac{B_0}{H} \right)\\
    &  =  \exp \left(- \frac{\alpha (3 / \alpha - 1)^2}{7/3 - 1 / \alpha} \frac{N w_\iota (\kappa + 1)}{H} \right).
\end{align}
\end{proof}
We are now ready to prove Lemma~\ref{lem:unbalanced_episodes_bound} by combining the bound in the previous lemma for all $\kappa$ and $\iota$.
\begin{proof}[\textbf{Proof of Lemma~\ref{lem:unbalanced_episodes_bound}}]
    Since $w_k(s,a) \leq H$, we have that $\frac{w_k(s,a)}{w_{\min}} <
    \frac{H}{w_{\min}} $ and so $\iota_k(s,a) \leq  H / w_{\min} = 4 H^2 \numS
    / \epsilon$. In addition, $|X_{k, \kappa, \iota}| \leq \numS$ for all
$k, \kappa, \iota$ and so $|X_{k, \kappa, \iota}| > \kappa$ can only be true
    for $\kappa \leq \numS$.  Hence, only 
    \[
        E_{\max} = \log_2 \frac{H}{w_{\min}}  \log_2 \numS 
    \]
    possible values for $(\kappa, \iota)$ exists that can have $|X_{k, \kappa, \iota}| > \kappa$.
    Using the union bound over all $(\kappa, \iota)$ and Lemma~\ref{lem:bound_ki_epis}, we get that 
    \begin{align}
        \prob ( E \leq  \alpha N E_{\max}) \geq & \prob ( \max_{(\kappa, \iota)} E_{\kappa, \iota} \leq \alpha N) \geq
        1 - E_{\max} \exp\left( - \frac{\beta w_\iota (\kappa + 1) N}{H} \right)\\
        \geq & 1 - E_{\max} \exp\left( - \frac{\beta w_{\min} N}{H} \right)
    = 1 - E_{\max} \exp\left( - \frac{\beta w_{\min} m |\saspace|}{H} \right)\\
    = & 1 - E_{\max} \exp\left( - \frac{\beta \epsilon m |\saspace|}{4 H^2 \numS} \right)
    \end{align}
    Bounding the right hand-side by $1 - \delta / 2$ and solving for $m$ gives
    \begin{align}
        1 - E_{\max} \exp\left( - \frac{\beta \epsilon m |\saspace|}{4 H^2 \numS} \right) \geq & 1 - \delta / 2\quad \Leftrightarrow \quad
        m \geq  \frac{4 H^2 \numS}{ |\saspace| \beta \epsilon} \ln \frac{2 E_{\max}}{\delta}
    \end{align}
    Hence, the condition 
    \begin{align}
        m \geq  \frac{4 H^2}{ \beta \epsilon} \ln \frac{2 E_{\max}}{\delta}
    \end{align}
    is sufficient for the desired result to hold. By plugging in $\alpha = 6$
    and $\beta = \frac{\alpha (3 / \alpha - 1)^2}{7/3 - 1 / \alpha} =
    \frac{9}{13} \geq \frac{2}{3}$, we obtain the statement to show.
\end{proof}

\subsection{Bound on the value function difference for episodes with $\forall \kappa, \iota: \, |X_{k, \kappa, \iota}| \leq \kappa$ }

To prove Lemma~\ref{lem:balanced_eps_good}, it is sufficient to consider a
fixed phase $k$. To avoid notational clutter, we therefore omit the phase
indices $k$ in this section.

For the proof, we reason about a sequence of MDPs $M_d$ which have the same
transition probabilities but different reward functions $r^{(d)}$. For $d=0$,
the reward function is the original reward function $r$ of $M$, i.e. $r^{(0)}_t
= r_t$ for all $t=1 \dots H$. The following reward functions are then defined
recursively as $r^{(2d+2)}_t = \sigma_{t:H}^{(d), 2}$, where
$\sigma_{t:H}^{(d), 2}$ is the local variance of the value function w.r.t. the
rewards $r^{(d)}$. Note that for every $d$ and $t =1\dots H$ and $s \in
\statespace$, we have $r_t^{(d)}(s) \in [0, H^d]$. In complete analogy, we
define $\tilde M_d$ and $\hat M_d$.

We first prove a sequence of lemmas necessary for Lemma~\ref{lem:balanced_eps_good}.
\begin{lem}
    \begin{equation}
        V_{i,j} - \tilde V_{i,j} = \sum_{t=i}^{H-1} P_{i:t-1}(P_t - \tilde P_t) \tilde V_{t+1:j}
    \end{equation}
    
    \label{lem:V_diff}
\end{lem}
\begin{proof}
    \begin{align}
        V_{i,j}(s) - \tilde V_{i,j}(s) = &
        r(s) + P_i V_{i+1:j}(s) - r(s) - \tilde P_i \tilde V_{i+1:j}(s) + P_i \tilde V_{i+1,j}(s) - P_i \tilde V_{i+1:j}(s)\\ 
        =&  
        P_i (V_{i+1:j} - \tilde V_{i+1:j}) + (P_i - \tilde P_i) \tilde V_{i+1:j}(s)
    \end{align}
    Since we have $V_{j:j}(s) = r(s) = \tilde V_{j:j}(s)$, we can recursively expand the first difference until $i=j$ and get 
    \begin{align}
        V_{i,j} - \tilde V_{i,j} = &
        \sum_{t=i}^{j-1} P_{i:t-1}(P_t - \tilde P_t) \tilde V_{t+1:j}
    \end{align}
\end{proof}

\begin{lem}
    Assume $p, \hat p, \tilde p \in [0,1]$ satisfy
    $p \in \mathcal P$ and $\tilde p \in \operatorname{conv}(\mathcal P)$ where
    \begin{align}
        \mathcal P \defeq \bigg\{ p' \in [0,1] :& |\hat p - p'| \leq  \sqrt{\frac{\ln(6 / \delta_1)}{2 n}}, \label{eqn:pcond1}\\
                                                & |\hat p - p'| \leq \sqrt{\frac{2 \hat p(1- \hat p)}{n}\ln(6 / \delta_1)} + \frac{7}{3(n-1)}\ln(6 / \delta_1),\label{eqn:pcond2}\\
                                            & \textrm{if } n > 1: \left|\sqrt{p' (1-p')} - \sqrt{\hat p (1 - \hat p)}\right| \leq \sqrt{\frac{2 \ln(6 / \delta_1)}{n-1}} \bigg\}.\label{eqn:pcond3}
    \end{align}
    Then
    \begin{align}
        | p - \tilde p | \leq & \sqrt{\frac{8 \tilde p (1 - \tilde p)}{n}\ln(6 / \delta_1)} + \frac{26}{3(n-1)} \ln(6 / \delta_1).
    \end{align}
    \label{lem:confset_pp_bound}
\end{lem}
\begin{proof}
    We have $\mathcal P = \mathcal P_1 \cap \mathcal P_2$ with
    \begin{align}
        \mathcal P_1 = \bigg\{ p' \in [0,1] :& |\hat p - p'| \leq  \sqrt{\frac{\ln(6 / \delta_1)}{2 n}},\\
                                                & |\hat p - p'| \leq \sqrt{\frac{2 \hat p(1- \hat p)}{n}\ln(6 / \delta_1)} + \frac{7}{3(n-1)}\ln(6 / \delta_1),\\
                                            & \textrm{if } n > 1: \left(\max \left\{0, \sqrt{\hat p (1 - \hat p)} -  \sqrt{\frac{2 \ln(6 / \delta_1)}{n-1}}\right\}\right)^2 \leq  p' (1-p') \bigg\}.
    \end{align}
   and 
    \begin{align}
        \mathcal P_2 = \bigg\{ p' \in \reals \, : \textrm{if } n > 1: \sqrt{p' (1-p')} \leq  \sqrt{\hat p (1 - \hat p)} + \sqrt{\frac{2 \ln(6 / \delta_1)}{n-1}} \bigg\}.
    \end{align}
    Note that the last condition of $\mathcal P_1$ is equivalent to $ \sqrt{\hat p (1 - \hat p)}  \leq \sqrt{p' (1-p')} + \sqrt{\frac{2 \ln(6 / \delta_1)}{n-1}} $ as $p' \in [0, 1]$.
    As an intersection of a polytope and the superlevel set of a concave function $p' ( 1 - p')$, the set $\mathcal P_1$ is convex. 
    Hence $\operatorname{conv}(\mathcal P) = \operatorname{conv}(\mathcal P_1 \cap \mathcal P_2) \subseteq
    \operatorname{conv}(\mathcal P_1) = \mathcal P_1$. It therefore follows that $\tilde p \in \mathcal P_1$. We now bound
    \begin{align}
        | p - \tilde p | \leq & | p - \hat p | + | \hat p - \tilde p| 
        \leq  
        2\sqrt{\frac{2 \hat p (1 - \hat p)}{n} \ln(6/ \delta_1)} 
        + 2\frac{7}{3(n-1)} \ln(6 / \delta_1)\\
        = &
        \sqrt{ \hat p (1 - \hat p)}\sqrt{\frac{8 }{n} \ln(6/ \delta_1)} 
        + \frac{14}{3(n-1)} \ln(6 / \delta_1)\\
        \leq &
        \left(\sqrt{ \tilde p (1 - \tilde p)} + \sqrt{\frac{2 \ln(6 / \delta_1)}{n-1}} \right) \sqrt{\frac{8 }{n} \ln(6/ \delta_1)} 
        + \frac{14}{3(n-1)} \ln(6 / \delta_1)\\
        \leq &
        \sqrt{\frac{8 \tilde p (1 - \tilde p)}{n} \ln(6/ \delta_1)} 
        + \frac{26}{3(n-1)} \ln(6 / \delta_1)
    \end{align}
\end{proof}

\begin{lem}
    Assume
    \begin{align}
        |p(s' | s,a) - \tilde p_i(s' | s,a)| \leq c_1(s,a) + c_2(s,a) \sqrt{\tilde p_i(s' | s,a) (1 - \tilde p_i(s' | s,a))}
    \end{align}
    for $a = \pi_i(s)$ and all $s', s \in \statespace$. Then
    \begin{align}
        |(P_i - \tilde P_i) \tilde V_{i+1:j}(s)| \leq c_1(s,a) \numSucc{s,a} \|  \tilde V_{i+1:j} \|_\infty + c_2(s,a) \sqrt{\numSucc{s,a}} \tilde \sigma_{i:j}(s)
    \end{align}
    for any $(s,a) \in \saspace$ where $\succS{s,a}$ denotes the set of possible successor states of state $s$ and action $a$.
    \label{lem:sigma_bound_generic}
\end{lem}
\begin{proof}
    Let $s$ and $a= \pi_i(s)$ be fixed and define for this fixed $s$ the constant function
    $\bar V(s') = \tilde P_i \tilde V_{i+1:j}(s)$ \emph{[sic]} as the expected
    value function of the successor states of $s$. Note that $\bar V(s')$ is a
    constant function and so $\bar V = \tilde P_i \bar V = P_i \bar V$. 
    \begin{align}
        & |(P_i - \tilde P_i) \tilde V_{i+1:j}(s)| 
        = |(P_i - \tilde P_i) \tilde V_{i+1:j}(s) + \bar V(s) - \bar V(s)| \\
        =& |(P_i - \tilde P_i) (\tilde V_{i+1:j} - \bar V)(s)|\\
        \leq & \sum_{s' \in \succS{s,a}} |p(s' | s,a) - \tilde p_i(s' | s, a)| |\tilde V_{i+1:j}(s') - \bar V(s')|
        \label{eqn:sigma_bound_trang}\\
        \leq & \sum_{s' \in \succS{s,a}} \left(c_1(s,a) + c_2(s,a) \sqrt{ \tilde p_i(s' | s,a) (1 - \tilde p_i(s' | s,a))}\right) |\tilde V_{i+1:j}(s') - \bar V(s')|\\
        \leq & \numSucc{s,a} c_1(s,a) \| \tilde V_{i+1:j} \|_\infty + c_2(s,a) \sum_{s' \in \succS{s,a}} \sqrt{ \tilde p_i(s' | s,a) (1 - \tilde p_i(s' | s,a)) (\tilde V_{i+1:j}(s') - \bar V(s'))^2}\\
        \leq & \numSucc{s,a} c_1(s,a) \| \tilde V_{i+1:j} \|_\infty + 
        c_2(s,a)  \sqrt{ \numSucc{s,a} \sum_{s' \in \succS{s,a}}\tilde p_i(s' | s,a) (1 - \tilde p_i(s' | s,a)) (\tilde V_{i+1:j}(s') - \bar V(s'))^2}
        \label{eqn:sigma_bound_cs}
        \\
        \leq & \numSucc{s,a} c_1(s,a) \| \tilde V_{i+1:j} \|_\infty + 
        c_2(s,a)  \sqrt{ \numSucc{s,a}\sum_{s' \in \succS{s,a}}\tilde p_i(s' | s,a) (\tilde V_{i+1:j}(s') - \bar V(s'))^2}\\
        = & \numSucc{s,a} c_1(s,a) \| \tilde V_{i+1:j} \|_\infty + 
        c_2(s,a) \sqrt{\numSucc{s,a}} \tilde \sigma_{i:j}(s)
    \end{align}
    In Inequality~\eqref{eqn:sigma_bound_trang}, we wrote out the definition of
    $P_i$ and $\tilde P_i$ and applied the triangle inequality.  We then
    applied the assumed bound and bounded $|\tilde V_{i+1:j}(s') - \bar V(s')|$
    by $\| V_{i+1:j} \|_\infty$ as all value functions are nonnegative.  In
    Inequality~\eqref{eqn:sigma_bound_cs}, we applied the Cauchy-Schwarz
    inequality and subsequently used the fact that each term is the sum is
    nonnegative and that $(1 - \tilde p_i(s'| s,a)) \leq 1$. The final equality
    follows from the definition of $\tilde \sigma_{i:j}$.
\end{proof}
\subsubsection{Bounding the difference in value function}
\begin{lem}
    Assume $M \in \mathcal M_k$. If $| X_{\kappa, \iota}| \leq \kappa$ for
    all $(\kappa, \iota)$.
    Then
    \begin{equation}
        |V_{1:H}^{(d)}(s_0) - \tilde V_{1:H}^{(d)}(s_0) | =: \Delta_d \leq \hat A_d + \hat B_d + \min \{ \hat C_d, \hat C_d' + \hat C'' \sqrt{\Delta_{2d+2}}\}
    \end{equation}
    where
    \begin{equation}
        \hat A_d = \frac{\epsilon}{4} H^d, \qquad \hat B_d = \frac{52 H^{d+1} \numki \maxNumSucc}{3m} \ln \frac{6}{\delta_1},
    \end{equation}
    and
    \begin{equation}
\hat C_d' = \sqrt{ \maxNumSucc \numki \frac{8}{m} H^{2d+2}\ln\frac{6}{\delta_1} }
\qquad
\hat C_d = \hat C_d' \sqrt{H}, 
        \qquad
\hat C'' = \sqrt{ \maxNumSucc \numki \frac{8}{m}\ln\frac{6}{\delta_1}}.
    \end{equation}
    \label{lem:Delta_bounds}
\end{lem}
\begin{proof}
    \begin{align}
        \Delta_d =& |V_{1:H}^{(d)}(s_0) - \tilde V_{1:H}^{(d)}(s_0) |
        = \left| \sum_{t=1}^{H-1}P_{1:t-1}(P_t - \tilde P_t) \tilde V^{(d)}_{t+1:H}(s_0) \right| \\
        \leq &  \sum_{t=1}^{H-1}P_{1:t-1}|(P_t - \tilde P_t) \tilde V^{(d)}_{t+1:H}|(s_0) \\
        = & \sum_{t=1}^{H-1}P_{1:t-1}  \left( \sum_{s,a \in \saspace} \indicator{s = \cdot, a = \pi_t(s)} |(P_t - \tilde P_t) \tilde V^{(d)}_{t+1:H}| \right) (s_0)\\
        = & \sum_{s,a \in \saspace}  \sum_{t=1}^{H-1}P_{1:t-1}  \left(  \indicator{s = \cdot, a = \pi_t(s)} |(P_t - \tilde P_t) \tilde V^{(d)}_{t+1:H}| \right)(s_0)\\
        = & \sum_{s,a \in \saspace}  \sum_{t=1}^{H-1}P_{1:t-1}  \left(  \indicator{s = \cdot, a = \pi_t(s)} |(P_t - \tilde P_t) \tilde V^{(d)}_{t+1:H}(s)| \right)(s_0)
    \end{align}
    The first equality follows from Lemma~\ref{lem:V_diff}, the second step
    from the fact that $V_{t+1:H} \geq 0$ and $P_{1:t-1}$ being non-expansive.
    In the third, we introduce an indicator function which does not change the
    value as we sum over  all $(s,a)$ pairs. The fourth step relies on the
    linearity of the $P_{i:j}$ operators. In the fifth step, we realize that
    $\indicator{s = \cdot, a = \pi_t(s)} |(P_t - \tilde P_t) \tilde
    V_{t+1:H}^{(d)}(\cdot)$ is a function that takes nonzero values only for input
    $s$. We can therefore replace the argument of the second term with $s$
    without changing the value. The term then becomes constant and by linearity
    of $P_{i:j}$, we can write

    \begin{align}
        & |V_{1:H}^{(d)}(s_0) - \tilde V_{1:H}^{(d)}(s_0)| = \Delta_d 
        \leq  \sum_{s,a \in \saspace}  \sum_{t=1}^{H-1} |(P_t - \tilde P_t) \tilde V^{(d)}_{t+1:H}(s)| (P_{1:t-1}  \indicator{s = \cdot, a = \pi_t(s)})(s_0)\\
        \leq & \sum_{s,a \notin X}  \sum_{t=1}^{H-1} \| \tilde V^{(d)}_{t+1:H} \|_\infty (P_{1:t-1}  \indicator{s = \cdot, a = \pi_t(s)})(s_0)\\
             & + \sum_{s,a \in X}  \sum_{t=1}^{H-1} |(P_t - \tilde P_t) \tilde V^{(d)}_{t+1:H}(s)| (P_{1:t-1}  \indicator{s = \cdot, a = \pi_t(s)})(s_0)\\
        \leq & \sum_{s,a \notin X}  \sum_{t=1}^{H-1} H^{d+1} (P_{1:t-1}  \indicator{s = \cdot, a = \pi_t(s)})(s_0)\\
             & + \sum_{s,a \in X}  \sum_{t=1}^{H-1} |(P_t - \tilde P_t) \tilde V^{(d)}_{t+1:H}(s)| (P_{1:t-1}  \indicator{s = \cdot, a = \pi_t(s)})(s_0)\\
        \leq & \sum_{s,a \notin X}  \sum_{t=1}^{H-1} H^{d+1} (P_{1:t-1}  \indicator{s = \cdot, a = \pi_t(s)})(s_0)\\
             & + \sum_{s,a \in X}  \sum_{t=1}^{H-1} \left|\numSucc{s, a} c_1(s,a) H^{d+1} + c_2(s,a) \sqrt{\numSucc{s, a}}\tilde \sigma^{(d)}_{t:H}(s,a) \right| (P_{1:t-1}  \indicator{s = \cdot, a = \pi_t(s)})(s_0)\\
    \leq & \sum_{s,a \notin X}  \sum_{t=1}^{H} H^{d+1} (P_{1:t-1}  \indicator{s = \cdot, a = \pi_t(s)})(s_0)\\
             & + \sum_{s,a \in X}  \sum_{t=1}^{H} \left|\numSucc{s, a} c_1(s,a) H^{d+1} \right| (P_{1:t-1}  \indicator{s = \cdot, a = \pi_t(s)})(s_0)\\
             & + \sum_{s,a \in X}  \sum_{t=1}^{H-1} \left|c_2(s,a) \sqrt{\numSucc{s, a}}\tilde \sigma^{(d)}_{t:H}(s,a) \right| (P_{1:t-1}  \indicator{s = \cdot, a = \pi_t(s)})(s_0)\\
        \leq & \sum_{s,a \notin X} H^{d+1} w(s,a) 
        + \sum_{s,a \in X} \numSucc{s, a} c_1(s,a) H^{d+1} w(s,a)\\
        & + \sum_{s,a \in X}  \sqrt{\numSucc{s, a}} c_2(s,a)\sum_{t=1}^{H-1}  \tilde \sigma^{(d)}_{t:H}(s,a) (P_{1:t-1}  \indicator{s = \cdot, a = \pi_t(s)})(s_0)\\
        \leq & \sum_{s,a \notin X} H^{d+1} w(s,a) 
        + \sum_{s,a \in X} \maxNumSucc c_1(s,a) H^{d+1} w(s,a)\\
        & + \sum_{s,a \in X}  \sqrt{\maxNumSucc} c_2(s,a)\sum_{t=1}^{H-1}  \tilde \sigma^{(d)}_{t:H}(s,a) (P_{1:t-1}  \indicator{s = \cdot, a = \pi_t(s)})(s_0)
    \end{align}
    In the second inequality, we split the sum over all $(s,a)$ pairs and used
    the fact that $P_t$ and $\tilde P_t$ are non-expansive, i.e., $|(P_t -
    \tilde P_t) \tilde V^{(d)}_{t+1:H}(s)| \leq \| V^{(d)}_{t+1:H} \|_\infty$.
    The next step follows from  $\| V^{(d)}_{t+1:H} \|_\infty \leq \|
    V^{(d)}_{1:H} \|_\infty \leq H^{d+1}$. We then apply
    Lemma~\ref{lem:sigma_bound_generic} and subsequently use that all terms are
    nonnegative and the definition of $w(s,a)$.  Eventually, we use that
    $\numSucc{s,a} \leq \maxNumSucc$ for all $s,a$.  Using the assumption that
    $M \in \mathcal M_k$ and $\tilde M \in \mathcal M'_k$ from
    Lemma~\ref{lem:planning}, we can apply Lemma~\ref{lem:confset_pp_bound} and
    get that
    \[
        c_2(s,a) = \sqrt{\frac{8}{n(s,a)} \ln \frac{6}{\delta_1}}  
    \quad
    \textrm{and}
    \quad
    c_1(s,a) = \frac{26}{3(n(s,a)-1)} \ln \frac{6}{\delta_1}.
    \]
    Hence, we can bound
    \[
        \Delta_d \leq A(s_0) + B(s_0) + C(s_0)
    \]
    as a sum of three terms which we will consider individually in the following.
    The first term is
    \[
        A(s_0) =  \sum_{s,a \notin X} H^{d+1} w(s,a) \leq w_{\min} \numS H^{d+1} 
        \leq \frac{\epsilon H^{d+1} \numS}{4 H \numS} = \frac{\epsilon}{4} H^{d} = \hat A_d
    \]
    as $w(s,a) \leq w_{\min}$ for all $s,a$ not in the active set
    and that the policy is deterministic, which implies that there are only
    $\numS$ nonzero $w$.  The next term is
    \begin{align}
        B(s_0) = & \maxNumSucc \sum_{s,a \in X} w(s,a) H^{d+1} \frac{26}{3 (n(s,a)-1)} \ln \frac{6}{\delta_1} \\
        = & H^{d+1} \maxNumSucc \ln \frac{6}{\delta_1} \sum_{\kappa, \iota} \sum_{s,a \in X_{\kappa, \iota}} w(s,a)  \frac{26}{3 (n(s,a)-1)}  \\
        \leq & H^{d+1} \frac{26\maxNumSucc}{3} \ln \frac{6}{\delta_1}  \sum_{\kappa, \iota}  \sum_{s,a \in X_{\kappa, \iota}} \frac{w(s,a)}{n(s,a) } \frac{n(s,a)}{n(s,a) - 1}.
    \end{align}
    For $s,a \in X_{\kappa, \iota}$, we have $n(s,a) \geq m w(s,a) \kappa$ (see Equation~\eqref{eqn:ki_n}) and so
    \begin{equation}
        \frac{w(s,a)}{n(s,a)} \leq \frac{1}{\kappa m}. \label{eqn:w_over_n_bound}
    \end{equation}
    Further, for all relevant $(s,a)$-pairs, we have $n(s,a) > 1$ (follows from $|X_{\kappa, \iota}| \leq \kappa$) which implies
    \begin{align}
        B(s_0) \leq & H^{d+1} \frac{52\maxNumSucc}{3} \ln \frac{6}{\delta_1}  \sum_{\kappa, \iota}  \frac{|X_{\kappa, \iota}|}{\kappa m}
    \end{align}
    and since we assumed $| X_{\kappa, \iota} | \leq \kappa$ 
    \begin{align}
        B(s_0)  
        \leq \frac{52 H^{d+1} \numki \maxNumSucc}{3 m} \ln \frac{6}{\delta_1} = \hat B_d
    \end{align}
    where $\kispace$ is the set of all possible $(\kappa, \iota)$-pairs.
     The last term is
     \begin{align}
         C(s_0) =& 
         \sqrt{\maxNumSucc} \sum_{s,a \in X}  c_2(s,a) \sum_{t=1}^{H-1} \tilde \sigma^{(d)}_{t:H}(s,a)) P_{1:t-1} \indicator{ s = \cdot, a = \pi_t (s) } \\
         \leq & \sqrt{\maxNumSucc} \sum_{s,a \in X} c_2(s,a) \sum_{t=1}^{H-1} \tilde \sigma^{(d)}_{t:H}(s,a)) P_{1:t-1} \indicator{ s = \cdot, a = \pi_t (s) } \\
         \leq & \sqrt{\maxNumSucc} \sum_{s,a \in X} c_2(s,a) 
         \sqrt{\sum_{t=1}^{H-1} P_{1:t-1} \indicator{ s = \cdot, a = \pi_t (s) }}
         \sqrt{\sum_{t=1}^{H-1} \tilde \sigma^{(d),2}_{t:H}(s,a)) P_{1:t-1} \indicator{ s = \cdot, a = \pi_t (s) }} \\
         \leq & \sqrt{\maxNumSucc} \sum_{s,a \in X}  
     \sqrt{ \frac{8 w(s,a)}{n(s,a)}\ln\frac{6}{\delta_1} \sum_{t=1}^{H-1} \tilde \sigma^{(d),2}_{t:H}(s,a)) P_{1:t-1} \indicator{ s = \cdot, a = \pi_t (s) }} \end{align}
     where we first applied the Cauchy-Schwarz inequality and then used the definition of $c_2(s,a)$ and $w(s,a)$.
     \begin{align}
         C(s_0)
         \leq & \sqrt{\maxNumSucc} \sum_{\kappa, \iota} \sum_{s,a \in X_{\kappa, \iota}}  
         \sqrt{ \frac{8 w(s,a)}{n(s,a)}\ln\frac{6}{\delta_1} \sum_{t=1}^{H-1} \tilde \sigma^{(d),2}_{t:H}(s,a)) P_{1:t-1} \indicator{ s = \cdot, a = \pi_t (s) }(s_0)} \\
         \leq & \sqrt{\maxNumSucc} \sum_{\kappa, \iota} 
         \sqrt{|X_{\kappa, \iota}| \sum_{s,a \in X_{\kappa, \iota}} \frac{8 w(s,a)}{n(s,a)}\ln\frac{6}{\delta_1} \sum_{t=1}^{H-1} \tilde \sigma^{(d),2}_{t:H}(s,a)) P_{1:t-1} \indicator{ s = \cdot, a = \pi_t (s) }(s_0)} \\
         \leq & \sqrt{\maxNumSucc} \sum_{\kappa, \iota} 
         \sqrt{\sum_{s,a \in X_{\kappa, \iota}} \frac{8}{m}\ln\frac{6}{\delta_1} \sum_{t=1}^{H-1} \tilde \sigma^{(d),2}_{t:H}(s,a)) P_{1:t-1} \indicator{ s = \cdot, a = \pi_t (s) }(s_0)} \\
         \leq  & \sqrt{\maxNumSucc \numki \frac{8}{m}\ln\frac{6}{\delta_1} \sum_{s,a \in X} \sum_{t=1}^{H-1} \tilde \sigma^{(d),2}_{t:H}(s,a)) P_{1:t-1} \indicator{ s = \cdot, a = \pi_t (s) }(s_0)} \\
         \leq & \sqrt{\maxNumSucc \numki \frac{8}{m}\ln\frac{6}{\delta_1} 
     \sum_{s,a \in \saspace} \sum_{t=1}^{H-1} \tilde \sigma^{(d),2}_{t:H}(s,a)) P_{1:t-1} \indicator{ s = \cdot, a = \pi_t (s) }(s_0)} \\
     = & \sqrt{\maxNumSucc \numki \frac{8}{m}\ln\frac{6}{\delta_1} 
 \sum_{t=1}^{H-1}  P_{1:t-1} \tilde \sigma^{(d),2}_{t:H} (s_0)}
         \label{eqn:D_bound_generic}\\
         \leq & \sqrt{ \maxNumSucc \numki  \frac{8 H^{2d+3}  \ln(6 / \delta_1)}{ m}} = \hat C_d
 \end{align}
 We first split the sum and applied the Cauchy-Schwarz inequality. Then we used
 again Inequality~\eqref{eqn:w_over_n_bound} and $|X_{\kappa, \iota}| \leq
 \kappa$.  In the fourth step, we applied Cauchy-Schwarz and the final
 inequality follows from $\| \tilde \sigma_{t:H}^{(d),2} \|_\infty \leq H^{2d +
 2}$ and the fact that $P_{1:t-1}$ is non-expansive.  Alternatively, we can
 rewrite the bound in Equation~\eqref{eqn:D_bound_generic} as
 \begin{align}
     C(s_0) 
     \leq &  \sqrt{\maxNumSucc \numki \frac{8}{m}\ln\frac{6}{\delta_1} 
 \sum_{t=1}^{H-1}  P_{1:t-1} \tilde \sigma^{(d),2}_{t:H} (s_0)} \\
     = &  \sqrt{ \maxNumSucc \numki \frac{8}{m}\ln\frac{6}{\delta_1} 
 \sum_{t=1}^{H-1}  P_{1:t-1} \tilde \sigma^{(d),2}_{t:H} (s_0) - \tilde P_{1:t-1} \tilde \sigma^{(d),2}_{t:H} (s_0) + \tilde P_{1:t-1} \tilde \sigma^{(d),2}_{t:H} (s_0)}.
 \end{align}
 Lemma~\ref{lem:varV_bellman} shows that the variance $\tilde
 \varV^{(d)}_{1:H}$ also satisfies the Bellman equation with the local
 variances $\tilde \sigma^{(d),2}_{i:j}$.  This insight allows us to bound
 $\sum_{t=1}^{H-1} \tilde P_{1:t-1} \tilde \sigma^{(d),2}_{t:H}(s_0) = \tilde
 \varV^{(d)}_{1:H}(s_0) \leq H^{2d+2}$. 
 Also, note that 
 $\tilde \sigma^{(d),2}_{t:H} = r_t^{(2d+2)}$ which gives us
 \begin{align}
     C(s_0) 
     \leq  & \sqrt{ \maxNumSucc \numki \frac{8}{m}\ln\frac{6}{\delta_1} 
 \left( H^{2d+2} + \sum_{t=1}^{H-1}  P_{1:t-1} r^{(2d+2)}_{t}(s_0) - \tilde P_{1:t-1} r^{(2d+2)}_{t} (s_0) \right)}\\
     = &  \sqrt{ \maxNumSucc \numki \frac{8}{m}\ln\frac{6}{\delta_1} 
 \left( H^{2d+2} + V_{1:H}^{(2d+2)}(s_0) - \tilde V_{1:H}^{(2d+2)}(s_0) \right)}\\
     \leq &  \sqrt{ \maxNumSucc \numki \frac{8}{m}\ln\frac{6}{\delta_1} 
 \left( H^{2d+2} + \Delta_{2d+2} \right)}\\
     \leq &  \sqrt{ \maxNumSucc \numki \frac{8}{m} H^{2d+2}\ln\frac{6}{\delta_1} }
 + \sqrt{ \maxNumSucc \numki \frac{8}{m}\Delta_{2d+2}\ln\frac{6}{\delta_1}} = \hat C_d' + \hat C'' \sqrt{\Delta_{2d+2}}
 \end{align}
 \end{proof}

 \subsubsection{Proof of Lemma~\ref{lem:varV_bellman} (Bellman equation of local value function variances)}
 \begin{proof}[Proof of Lemma~\ref{lem:varV_bellman}]
    \begin{align}
        \varV_{i:j}(s) 
        = & \mathbb E\left[ \left( \sum_{t=i}^{j} r_t(s_t) - V_{i:j}(s_i) \right)^2 | s_i = s \right] \\
        = & \mathbb E\left[ \left( \sum_{t=i+1 }^{j} r_t(s_t) - V_{i+1:j}(s_{i+1}) +
V_{i+1:j}(s_{i+1}) + r_i(s_i) - V_{i:j}(s_i) \right)^2 | s_i = s \right] \\
= & \mathbb E\left[ \left( \sum_{t=i+1 }^{j} r(s_t) - V_{i+1:j}(s_{i+1})\right)^2 | s_i = s \right] \\
        &+
        2 \mathbb E\left[ \left( \sum_{t=i+1 }^{j} r_t(s_t) - V_{i+1:j}(s_{i+1})\right) \left(V_{i+1:j}(s_{i+1}) + r_i(s_i) - V(s_i) \right) |  s_i = s \right] \\
        & + \mathbb E \left[\left(V_{i+1:j}(s_{i+1}) + r_i(s_i) - V_{i:j}(s_i) \right)^2 | s_i = s \right]\\
        = & \mathbb E\left[ \varV_{i+1:j}(s_{i+1}) | s_i = s \right] \\
        &+
        2 \mathbb E\left[ \mathbb E \left[ \left( \sum_{t=i+1 }^{j} r_t(s_t) - V_{i+1:j}(s_{i+1})\right) \left(V_{i+1:j}(s_{i+1}) + r_i(s_i) - V_{i:j}(s_i) \right) | s_{i+1} \right] |  s_i = s \right] \\
        & + \mathbb E \left[\left(V_{i+1:j}(s_{i+1}) - P_i V_{i+1:j}(s_i) \right)^2 | s_i = s \right]
    \end{align}
    where the final equality follows from the tower property of conditional expectations, and the fact that $V_{i:j}(s_i) = P_i V_{i+1:j}(s_i) + r_i(s_i)$. Since by the definition of the value function
    \[ 
        \mathbb E \left[ \left( \sum_{t=i+1 }^{j} r_t(s_t) - V_{i+1:j}(s_{i+1})\right) | s_{i+1} \right] = 0
    \]
    the middle term vanishes and the last term is by definition $\sigma_{i:j}^2 (s)$ we obtain
    \begin{align}
        \varV_{i:j}(s)  = P_i \varV_{i+1:j}(s) + \sigma_{i:j}^2(s).
    \end{align}
    Noting that $\varV_{j:j}(s) = (r_j(s) - r_j(s))^2 = 0$, we can unroll the equation and obtain
    \begin{equation}
        \varV_{i:j}(s) = \sum_{t=i}^j P_{i:t-1} \sigma^2_{t:j}(s).
    \end{equation}
    From the definition of $\varV_{1:H}$ and the fact that $0 \leq r(\cdot) \leq
    r_{\max}$, we see that $0 \leq \varV_{1:H} \leq H^2 r_{\max}^2$ and the final statement of
    the lemma follows.

\end{proof}
\subsubsection{Proof of Lemma~\ref{lem:balanced_eps_good}}
\begin{proof}[Proof of Lemma~\ref{lem:balanced_eps_good}]
     The recursive bound from Lemma~\ref{lem:Delta_bounds}
     \[
         \Delta_d \leq \hat A_d + \hat B_d + \hat C_d' + \hat C'' \sqrt{\Delta_{2d+2}}
 \]
 has the form $\Delta_d \leq Y_d + Z \sqrt{\Delta_{2d+2}}$. Expanding this form and using the triangle inequality gives
 \begin{align}
     \Delta_0 \leq & Y_0 + Z \sqrt{\Delta_{2}} \leq Y_0 + Z \sqrt{Y_2 + Z \sqrt{\Delta_{6}}}
     \leq Y_0 + Z \sqrt{Y_2} + Z^{3/2} \Delta_{6}^{1/4}\\
     \leq & Y_0 + Z \sqrt{Y_2} + Z^{3/2} Y_6^{1/4} + Z^{7/4} \Delta_{14}^{1/8} \leq \dots
 \end{align}
 and by doing this up to level $\gamma = \lceil \frac{\ln H}{2 \ln 2} \rceil$, we obtain
 \begin{align}
     \Delta_0 \leq \sum_{d \in \mathcal D \setminus \{ \gamma \}} Z^{\frac{2d}{2 + d}} Y_d^{\frac{2}{2+d}} + Z^{\frac{2\gamma}{2 + \gamma}} \Delta_\gamma ^{\frac{2}{2 + \gamma}}
 \end{align}
 where
 $\mathcal D = \{ 0, 2, 6, 14, \dots \gamma \}$.
Note that the exponent of $H$ compared to $m$ is the larger in $\hat C_d'$ than in $\hat B_d$. Therefore, for sufficiently large $m$, $\hat C_d'$ dominates the other term. More precisely, for
\begin{align}
    m \geq \frac{338H}{9} \maxNumSucc \numki \ln \frac{6}{\delta_1}
    \label{eqn:m_cond_ABC}
\end{align}
we have $\hat B_d \leq \hat C_d'$.
We can therefore consider
$Z = \hat C''$ and $Y_d = 2 \hat C_d' + \hat A_d$. Also, since $\hat C_d \geq \hat C'_d$, we can bound $\Delta_\gamma \leq \hat A_d + 2 \hat C_d$.
For notational simplicity, we will use the auxiliary variable
\[
    m_1 = \frac{8 \maxNumSucc | \mathcal K \times \mathcal I| H^2 }{m \epsilon^2} \ln \frac{6}{\delta_1}.
\]
and get
\begin{align}
    Z &= \hat C'' = \sqrt{m_1}\frac{\epsilon}{H} \quad \textrm{and}\\
    Y_d &= \hat A_d + 2 \hat C'_d = (1/4 + 2 \sqrt{m_1}) H^d \epsilon\quad \textrm{and}\\
    \Delta_\gamma & \leq \hat A_\gamma + 2 \hat C_\gamma = (1/4 + 2 \sqrt{m_1 H}) H^\gamma \epsilon.
\end{align}
Then 
\begin{align}
    \left(Z^{2d} Y_d^{2} \right)^{(2+d)^{-1}} 
    =& \left(m_1^d \epsilon^{2d+2} (1/4 + 2 \sqrt{m_1})^2 \right)^{(2+d)^{-1}}
    = \epsilon \left(m_1^d \epsilon^{d} (1/4 + 2 \sqrt{m_1})^2 \right)^{(2+d)^{-1}}
\end{align}
and
\begin{align}
    \left(Z^{2\gamma} \Delta_\gamma \right)^{(2+\gamma)^{-1}} 
    =& \left(m_1^\gamma \epsilon^{2\gamma+2} (1/4 + 2 \sqrt{m_1 H})^2 \right)^{(2+\gamma)^{-1}}
    = \epsilon  \left(m_1^\gamma \epsilon^{\gamma} (1/4 + 2 \sqrt{m_1 H})^2 \right)^{(2+\gamma)^{-1}}.
\end{align}
Putting these pieces together, we obtain
\begin{align}
    \frac{\Delta_0}{\epsilon} \leq & \sum_{d \in \mathcal D \setminus \{ \gamma \}} (\epsilon m_1)^{\frac{d}{2+d}} \left( \frac{1}{4} + 2 \sqrt{m_1} \right)^{\frac{2}{d + 2}}
    + (\epsilon m_1)^{\frac{\gamma}{\gamma + 2}} \left(\frac{1}{4} + 2 \sqrt{H m_1} \right)^{\frac{2}{\gamma + 2}}
    \\
    = & 
    \frac{1}{4} + 2 \sqrt{m_1} + 
    \sum_{d \in \mathcal D \setminus \{0, \gamma \}} (\epsilon m_1)^{\frac{d}{2+d}} \left( \frac{1}{4} + 2 \sqrt{m_1} \right)^{\frac{2}{d + 2}}
    + (\epsilon m_1)^{\frac{\gamma}{\gamma + 2}} \left(\frac{1}{4} + 2 \sqrt{H m_1} \right)^{\frac{2}{\gamma + 2}}\\
    \leq & 
    \frac{1}{4} + 2 \sqrt{m_1} + 
    \sum_{d \in \mathcal D \setminus \{0, \gamma \}}(\epsilon m_1)^{\frac{d}{2+d}} \left[ \left( \frac{1}{4} \right)^{\frac{2}{d + 2}} 
      + \left(  2 \sqrt{m_1} \right)^{\frac{2}{d + 2}} \right] \\
      & + (\epsilon m_1)^{\frac{\gamma}{\gamma + 2}} \left[ \left(\frac{1}{4} \right)^{\frac{2}{\gamma + 2}}+ \left( 2 \sqrt{H m_1} \right)^{\frac{2}{\gamma + 2}}\right]
\end{align}
where we used the fact that $(a+b)^\phi \leq a^\phi + b^\phi$ for $a,b >0$ and $0 < \phi <1$.
We now bound the $H^{1/(2+\gamma)}$ by using the definition of $\gamma$. Since
\[
    \frac{1}{2 + \gamma} = \frac{2 \ln 2}{4 \ln 2 + \ln H} \leq 2 \log_H 2
\] 
and since $H \geq 1$, we have $H^{1/(2+\gamma)} \leq 4$.
Therefore
\begin{align}
    \frac{\Delta_0}{\epsilon} 
    \leq & 
    \frac{1}{4} + 2 \sqrt{m_1} + 
    \sum_{d \in \mathcal D \setminus \{0, \gamma \}}(\epsilon m_1)^{\frac{d}{2+d}} \left[ \left( \frac{1}{4} \right)^{\frac{2}{d + 2}} 
      + \left(  2 \sqrt{m_1} \right)^{\frac{2}{d + 2}} \right] \\
      & + (\epsilon m_1)^{\frac{\gamma}{\gamma + 2}} \left[ \left(\frac{1}{4} \right)^{\frac{2}{\gamma + 2}}+ 4 \left( 2 \sqrt{m_1} \right)^{\frac{2}{\gamma + 2}}\right]\\
    \leq & 
    \frac{1}{4} + 2 \sqrt{m_1} + 
    \sum_{d \in \mathcal D \setminus \{0 \}}(\epsilon m_1)^{\frac{d}{2+d}} \left[ \left( \frac{1}{4} \right)^{\frac{2}{d + 2}} 
      + 4\left(  2 \sqrt{m_1} \right)^{\frac{2}{d + 2}} \right]\\
    \leq & 
    \frac{1}{4} + 2 \sqrt{m_1} + 
    \sum_{i=1}^{\log_2 \gamma}(\epsilon m_1)^{1 - 2^{-i}} \left[ \left( \frac{1}{4} \right)^{2^{-i}} 
    + 4\left(  2 \sqrt{m_1} \right)^{2^{-i}} \right]\\
    \leq & 
    \frac{1}{4} + 2 \sqrt{m_1} + 
    \sum_{i=1}^{\log_2 \gamma} m_1^{1 - 2^{-i}} \left[ \left( \frac{1}{4} \right)^{2^{-i}} 
    + 4\left(  2 \sqrt{m_1} \right)^{2^{-i}} \right]
\end{align}
In the first inequality, we used the bound for $H^{1/(2+\gamma)}$ and in the
second inequality we simplified the expression by noting that all terms are
nonnegative. In the next step, we re-parameterized the sum. In the final
inequality, we used the assumption that $0 < \epsilon \leq 1$ and therefore
$\epsilon^{1 - 2^{-i}} \leq 1$.
\begin{align}
    \frac{\Delta_0}{\epsilon} 
    \leq & 
    \frac{1}{4} + 2 \sqrt{m_1} + 
    \frac{1}{4}\sum_{i=1}^{\log_2 \gamma} (4m_1)^{1 - 2^{-i}}  
    + 4 \sum_{i=1}^{\log_2 \gamma} (m_1)^{1 - 2^{-i}}  \left(  4 m_1 \right)^{2^{-i-1}} \\
    \leq & 
    \frac{1}{4} + 2 \sqrt{m_1} + 
    \frac{1}{4}\sum_{i=1}^{\log_2 \gamma} (4m_1)^{1 - 2^{-i}}  
    + 16 \sum_{i=1}^{\log_2 \gamma} \left(\frac{m_1}{4}\right)^{1 - 2^{-i-1}}.
\end{align}
By requiring that
\begin{equation}
    m_1 \leq \frac{1}{4}
\end{equation}
and noting that $1 - 2^{-i} \geq 1/2$ and $1 - 2^{-i-1} \geq 3/4$ for $i \geq 1$, we can bound the expression by
\begin{align}
    \frac{\Delta_0}{\epsilon} 
    \leq & 
    \frac{1}{4} + 2 \sqrt{m_1} + 
    \frac{1}{4} \log_2(\gamma) \sqrt{4m_1}  
    + 16 \log_2(\gamma) \left(\frac{m_1}{4}\right)^{3/4}.
\end{align}
By requiring that $m_1 \leq 1 / 64$ and $m_1 \leq (2 \log_2 \gamma)^{-2}$ and
$m_1 \leq 1/ 64 (\log_2 \gamma)^{-4/3}$, we can assure that $\Delta_0
\leq \epsilon$.
Taking all assumptions on $m_1$ we made above together, we realize that 
\begin{align}
    m_1 \leq \left(\frac{1}{8 \log_2 \log_2 H}\right)^2 \leq
 \left(\frac{1}{8 \log_2 \gamma}\right)^2
\end{align}
is sufficient for them to hold where we used $\log_2 \gamma = \log_2( \lceil \frac{1}{2} \log_2 H \rceil) \leq \log_2 \log_2 H$. 
This gives the following condition on $m$
\begin{align}
    m \geq 512 \maxNumSucc (\log_2 \log_2 H)^2 | \mathcal K \times \mathcal I| \frac{H^2}{\epsilon^2} \ln \frac{6}{\delta_1} 
\end{align}
which is a stronger condition that the one in Equation~\eqref{eqn:m_cond_ABC}.

By construction of $\iota(s,a)$, we have $\iota(s,a) \leq 2\frac{H}{w_{\min}} =
\frac{8 \numS H^2}{ \epsilon} = \frac{8 H^2 \numS}{\epsilon}$.  Also,
$\kappa_k(s,a) \leq \frac{\numS m H}{m w_{\min}} = \frac{ 4 \numS^2
H^2}{\epsilon}$. Therefore
\begin{align}
     \numki \leq \log_2 \frac{ 4 \numS^2  H^2}{\epsilon} \log_2 \frac{8 H^2 \numS}{\epsilon} \leq \log_2^2 \frac{8 H^2 \numS^2}{\epsilon}
 \end{align}
which let us conclude that 
\begin{align}
    m \geq 512 \frac{\maxNumSucc H^2}{\epsilon^2}(\log_2 \log_2 H)^2 \log_2^2 \left(\frac{8 H^2 \numS^2}{\epsilon}\right)  \ln \frac{6}{\delta_1} 
\end{align}
is a sufficient condition and thus, the statement to show, holds.
\end{proof}

\subsection{Proof of Theorem~\ref{thm:upper_bound}}

\begin{proof}[Proof of Theorem~\ref{thm:upper_bound}]
    By Lemma~\ref{lem:unbalanced_episodes_bound}, we know that the number of
    episodes where $| X_{\kappa, \iota}| > \kappa$ for some $\kappa, \iota$ is
    bounded by $6 E_{\max} | \saspace| m$ with probability at least $1 -
    \delta / 2$. For all other episodes, we have by
    Lemma~\ref{lem:balanced_eps_good} that $| \tilde R^{\pi_k} - R^{\pi_k} | < \epsilon$. Since, with probability at least $1 - \delta / 2$, we have by Lemma~\ref{lem:mdp_capture}
    $M \in \mathcal M_k$, we can use Lemma~\ref{lem:planning} which gives $\tilde R^{\pi_k} > R^* \geq R^{\pi_k}$ to conclude that with probabilty at least $1 - \delta / 2$, for all episodes with
    $| X_{\kappa, \iota} | \leq \kappa$ for all $\kappa, \iota$, we have $R^* - R^{\pi_k} < \epsilon$. Applying the union bound, we get the desired result, if $m$ satisfies
    \begin{align}
        m \geq & 512 \frac{\maxNumSucc H^2}{\epsilon^2}(\log_2 \log_2 H)^2 \log_2^2 \left(\frac{8 H^2 \numS^2}{\epsilon}\right)  \ln \frac{6}{\delta_1} 
    \quad \textrm{and} \\
        m \geq &  \frac{6 H^2}{ \epsilon} \ln \frac{2 E_{\max}}{\delta}.
  \end{align}
From the definitions, we get
\[
\ln \frac{6}{\delta_1} = \ln \frac{6 \maxNumSucc \maxUpdates}{\delta} 
= \ln \frac{6 |\saspace| \maxNumSucc \log_2 (\numS H / w_{\min})}{\delta}
= \ln \frac{6 |\saspace| \maxNumSucc \log_2 (4 \numS^2 H^2 / \epsilon)}{\delta}
\]
and
\[
    E_{\max} = \log_2 \numS \log_2 \frac{4 H^2 \numS}{\epsilon} \leq \log_2^2 \frac{4 H^2 \numS}{\epsilon} 
\]
and
\begin{align}
    \ln \frac{2 E_{\max}}{\delta} =& \ln \frac{2 \log_2 \numS \log_2 (4 H^2 \numS / \epsilon)}{\delta} 
    \leq \ln  \frac{2\log^2_2 (4 H^2 \numS / \epsilon)  }{\delta} \\
    \leq & \ln \frac{6 \numsa \log^2_2 (4 \numS^2 H^2 / \epsilon)}{\delta}.
\end{align}
Setting
\[
    m = 512 (\log_2 \log_2 H)^2 \frac{\maxNumSucc H^2}{\epsilon^2} \log^2 \left(\frac{8 H^2 \numS^2}{\epsilon} \right) 
        \ln \frac{6 \numsa \maxNumSucc \log^2_2 (4 \numS^2 H^2 / \epsilon)}{\delta}
\]
is therefore a valid choice for $m$ to ensure that with probability at least $1 - \delta$ , there are at most 

\begin{align}
    6 m E_{\max} =& 
    3072 (\log_2 \log_2 H)^2 \frac{\maxNumSucc H^2 \numsa}{\epsilon^2}\\
    & \times \log_2^2 \left( \frac{4 H^2 \numS}{\epsilon}\right) \log^2 \left(\frac{8 H^2 \numS^2}{\epsilon} \right) \ln \frac{6 |\saspace| \maxNumSucc \log^2_2 (4 \numS^2 H^2 / \epsilon)}{\delta}
\end{align}

$\epsilon$-suboptimal episodes.

\end{proof}

\section{Proof of the Lower PAC Bound}
\label{sec:lower_proofs}
\begin{proof}[Proof of Theorem~\ref{thm:lower_bound}]

We consider the class of MDPs shown in Figure~\ref{fig:finite_hard_mdp}. The MDPs
essentially consist of $n$ parallel multi-armed bandits. For each bandit, there
exist $m+1 = \numA$ possible instantiations, which we denote by $I_i = 0
\dots m$.  The instantiation, or \emph{hypothesis}, $I_i = 0$ corresponds to
$\epsilon_i(a) = \indicator{a = a_0} \epsilon' / 2$, that is, only action $a_0$
has a small bias. The other hypotheses $I_i = j$ for $j=1\dots m$ correspond to
$\epsilon_i(a) = \indicator{a = a_0} \epsilon' / 2 + \indicator{a = a_j}
\epsilon'$. We use $I = (I_1, \dots I_n)$ to indicate the instance of the
entire MDP.

We define $G_i = \{ \omega \in \Omega : \pi(i) = a_{I_i} \}$, the event that
$\pi$, the policy generated by $A$ chooses optimally in bandit $i$.  For a
given instance $I$, the difference between the optimal expected cumulative
reward $R^*_I$ and the expected cumulative reward $R^\pi_I$ of policy $\pi$ is at least
\[
    R^*_I - R^\pi_I \geq (H-2) \left(1 - \frac 1 n \sum_{i=1}^n \indicator{G_i}\right) \frac {\epsilon'} 2.
\]
For $\pi$ to be $\epsilon$-optimal, we therefore need
\begin{align}
    \epsilon \geq & R^*_I - R^\pi_I \geq (H-2) \left(1 - \frac 1 n \sum_{i=1}^n \indicator{G_i}\right) \frac {\epsilon'} 2, \\
    \frac{2\epsilon}{(H-2)\epsilon'} \geq & \left(1 - \frac 1 n \sum_{i=1}^n \indicator{G_i}\right), \\
   \frac 1 n \sum_{i=1}^n \indicator{G_i}  \geq & \left(1 - \frac{2\epsilon}{(H-2)\epsilon'}\right), \\
    \frac 1 n \sum_{i=1}^n \indicator{G_i}  \geq & \left(1 - \frac{2\epsilon (H-2) \eta}{(H-2) 16 \epsilon e^4}\right) = 1 - \frac{\eta}{8e^4}
\end{align}
where we chose value $\epsilon' \defeq \frac{16 \epsilon e^4}{(H-2)\eta}$ for
$\epsilon'$. We will specify the exact value of parameter $\eta$ later. 
The condition basically states that at least a fraction of $\phi
\defeq 1 - \eta / (8 e^4)$ bandits need to be solved optimally by $A$ for the
resulting policy $\pi$ to be $\epsilon$-accurate. For $A$ to be $(\epsilon,
\delta)$-correct, we therefore need 
\begin{equation}
    \prob_I\left( \frac 1 n \sum_{i=1}^n \indicator{G_i} \geq \phi \right)\geq \prob_I(R^*_I - R^\pi_I \geq \epsilon) \geq 1 - \delta
\end{equation}
for each instance $I$. Using Markov's inequality, we obtain
\begin{align}
    1 - \delta \leq  \prob_I\left( \frac 1 n \sum_{i=1}^n \indicator{G_i} \geq \phi\right)
        \leq \frac{1}{n\phi}\sum_{i=1}^n \mathbb E_I[\indicator{G_i}]
        \leq \frac{1}{n\phi}\sum_{i=1}^n \prob_I(G_i)
\end{align}

All $G_i$ are independent of each other by construction of the MDP. In fact 
$ \sum_{i=1}^n  \indicator{G_i}$ is Poisson-binomial distributed as $\indicator{G_i}$ are independent Bernoulli random variables with potentially different mean.
Therefore,
upper bounds $\delta_i$ must exist such that $\delta_i \geq P_I(G_i^C)$ for all
hypotheses $I$ and such that $1 - \delta \leq \frac{1}{n\phi} \sum_{i=1}^n (1 -
\delta_i)$ or equivalently
    $n(1 + \delta \phi - \phi) \geq \sum_{i=1}^n \delta_i$.
Since all $G_i$ are independent of each other and
\begin{equation}
    \epsilon' =  \frac{16 \epsilon e^4}{(H-2)\eta} \leq \frac{16 (H-2) e^4 \eta}{(H-2)64 e^4 \eta} = \frac 1 4
\end{equation}
we can apply Theorem~1 by \citet{Mannor2004} 
in cases where 
\begin{equation}
    \delta_i \leq \frac{1}{\eta}(1 - \phi + \delta \phi) \leq \frac{1}{\eta}(1 - \phi + \delta) 
    \leq  \frac{1}{8e^4} + \frac{\delta}{\eta} \leq \frac{2}{8e^4}.
\end{equation}
This result gives us the minimum
expected number of times $\mathbb E_I[n_i]$ we need to observe state $i$ to ensure that $P_I(G_i^C) \leq \delta_i$
\begin{equation}
    \mathbb E_I[n_i] \geq \left[\frac{c_{1}(\numA - 1)}{\epsilon'^2} 
    \ln\left( \frac{c_{2}}{\delta_i} \right)\right] \indicator{\eta \delta_i \leq 1 - \phi + \phi \delta},
\end{equation}
for appropriate constants $c_1$ and $c_2$ (e.g. $c_1 = 400$ and $c_2 = 4$).  We can find a valid lower bound
for the total number of samples for any $\delta_1, \dots \delta_n$ by
considering the worst bound over all $\delta_1, \dots \delta_n$. The following
optimization problem encodes this idea 
\begin{align}
    \label{eqn:bound_prob}
    \min_{\delta_1, \dots \delta_n} & \sum_{i=1}^n \ln \frac{1}{\delta_i} 
    \indicator{\eta \delta_i \leq 1 - \phi + \phi \delta}\\
    \textrm{s.t.} \,& \sum_{i=1}^n \delta_i  \leq n(1 + \phi \delta - \phi)
\end{align}
As shown in Lemma~\ref{lem:delta_dis} in the supplementary material, the optimal solution of the
optimization problem in Equation~\eqref{eqn:bound_prob} is $\delta_1 = \dots =
\delta_n = c$ if $\eta  (1 - \ln c) \leq 1$ with $c = 1 + \delta \phi - \phi$.
Since the left-hand side of this condition is decreasing in $c$, we can plug in
a lower bound of $c \geq 1 - \phi = \frac{\eta}{8 e^4}$ and get the sufficient
condition
\begin{equation}
    \eta ( 1 - \ln \frac{\eta}{8 e^4}) = \eta(1 - \ln \eta + 4  + \ln 8) \leq 1.
\end{equation}
It is easy to verify that $\eta = 1 / 10$ satisfies this condition. Hence $\delta_1 = \dots = \delta_n = c$ is the optimal solution to the problem in Equation~\eqref{eqn:bound_prob}.
In each episode, we only observe a single state $i$ and therefore, there need to be at least
\[
    \mathbb E_I[n_A] \geq \sum_{i=1}^n \mathbb E_I[n_i] \geq 
    \frac{c_1 (\numA -1) n}{\epsilon'^2}
    \ln\left( \frac{c_2}{\delta_i} \right) 
    \geq 
    \frac{c_1 (\numA -1) n}{\epsilon'^2}
    \ln\left( \frac{c_2}{\delta + \frac{\eta}{8e^4}} \right) 
\]
observed episodes for appropriate constants $c_1$ and $c_2$. Plugging in $\epsilon'$ and $n =
\numS - 3$, we obtain the desired statement.

\end{proof}

\begin{lem}
The optimization problem
\begin{align}
    \min_{\delta_1 \dots \delta_n \in [0, 1]} &\sum_{i=1}^n  \ln \frac 1 \delta_i 
    \indicator{\eta \delta_i \leq c}\\
    \textrm{s.t.} \,& \sum_{i=1}^n \delta_i \leq n c
\end{align}
with $c \in [0, 1]$ and 
\[
    \eta (1 - \ln c) \leq 1 
\]
has optimal solution 
$\delta_1 = \dots = \delta_n = c$.
\label{lem:delta_dis}
\end{lem}

\begin{proof}
Without the indicator part in the objective, we can show that $\delta_1 = \dots
= \delta_n = c$ is an optimal solution by checking the KKT conditions and
noting that the problem is convex. 
Let $k$ denote the number of $\delta_j$ that are set such that the indicator
function is $0$.  Without loss of generality we can assume that their value is
$\delta_P \defeq c / \eta$ and the remaining $\delta_j$ take the same value
$\delta_A$ (for a fixed $\delta_P$ and $k$, the problem reduces to the one
without the indicator functions).  Then the problem transforms into
\begin{align}
    \min_{\delta_A \in (0, 1), k \in \{0, 1, \dots n\}}&  (n - k) \ln \frac{1}{\delta_A}\\
                     & (n-k) \delta_A + k \delta_P \leq n c
\end{align}
We can rewrite the constraint as 
\begin{align}
    (n-k) \delta_A + k \delta_P & \leq n c\\
    (n-k) \delta_A & \leq nc - k \delta_P = \left(n - \frac{k}{\eta}\right) c\\
    \delta_A & \leq \frac{n - \frac k \eta}{n - k} c.
\end{align}
Since the objective decreases with $\delta_A$, it is optimal to choose $\delta_A$ as large as possible.
The optimization problem then reduces to
\begin{align}
    \min_{k \in \{0, \dots \lfloor{n / \gamma} \rfloor \}}& (n - k) \ln \left( \frac{n - k}{n - \gamma k} c^{-1} \right).
\end{align}
where we used for convenience $\gamma \defeq 1 / \eta$. 
We want to show that the optimal solution to this problem is $k=0$. 
We can therefore relax the problem to the continuous domain without loss of generality
\begin{align}
\min_{k \in [0,n / \gamma]}& (n - k) \ln \left( \frac{n - k}{n - \gamma k} c^{-1} \right).
\end{align}
By reparameterizing the problem with $\alpha = k / n$, we get
\begin{align}
    \min_{\alpha \in [0, 1 / \gamma]}& n (1 - \alpha) \ln \left( \frac{1 - \alpha}{c(1 - \gamma \alpha)} \right).
\end{align}
We realize that the minimizer does not depend on $n$ (while the value does).
The second derivative of the objective function is
\begin{align}
    n\frac{(\gamma - 1)^2}{(1 - \gamma \alpha)^2 (1 - \alpha)},
\end{align}
which is nonnegative for $\alpha \in [0, 1 / \gamma]$. 
Hence, the objective is convex in the feasible region and the
minimizer of this problem is $\alpha = 0$ if the derivative of the objective is
nonnegative in $0$.  The derivative of the objective in $0$ is given by
\begin{align}
    n(\gamma - 1 + \ln(c)).
\end{align}
A sufficient condition for $\alpha = 0$ being optimal is therefore
\begin{align}
    \gamma \geq 1 - \ln c
\end{align}
or, in terms of the original problem with $\eta = 1 / \gamma$, $\delta_1 = \dots \delta_n = c$ is optimal if
\begin{align}
    \eta ( 1 - \ln c)  \leq 1
\end{align}
\end{proof}

\end{document}